\documentclass[journal,10pt,twocolumn]{IEEEtran}
\usepackage{amsmath,amssymb,amsthm}
\usepackage{mathtools}
\usepackage{bm}
\usepackage{algorithm}
\usepackage{algorithmic}
\usepackage{cite}
\usepackage{hyperref}
\usepackage{graphicx}
\usepackage{booktabs}
\usepackage{multirow}
\usepackage{array}
\usepackage{enumerate}
\usepackage{tabularx}
\usepackage{color}
\usepackage[capitalize]{cleveref}
\usepackage[T1]{fontenc}
\usepackage[utf8]{inputenc}
\usepackage{amsmath,amssymb,bm}
\usepackage{graphicx}
\usepackage{enumerate}
\usepackage{booktabs}
\usepackage{multirow}
\usepackage{makecell}
\usepackage[table]{xcolor}
\usepackage{siunitx}
\usepackage{url}
\usepackage{cite}
\usepackage{balance}
\graphicspath{{figures/}}
\usepackage{subcaption}

\newtheorem{theorem}{Theorem}
\newtheorem{lemma}[theorem]{Lemma}
\newtheorem{corollary}[theorem]{Corollary}
\newtheorem{proposition}[theorem]{Proposition}
\newtheorem{definition}{Definition}
\newtheorem{remark}{Remark}
\newtheorem{assumption}{Assumption}

\newcommand{\prel}{\boldsymbol{p}_{\mathrm{rel}}}
\newcommand{\vrel}{\boldsymbol{v}_{\mathrm{rel}}}

\newcommand{\F}{\mathcal{F}}
\newcommand{\Cset}{\mathcal{C}^{*}}

\definecolor{cellgreen}{RGB}{220,240,220}
\definecolor{hdrblue}{RGB}{30,77,120}
 
\newcommand{\Kmax}{\kappa_{\max}}

\newcommand{\nadv}{n_{\mathrm{adv}}}

\newcommand{\eobs}{\hat{\boldsymbol{e}}_{\mathrm{obs}}}
\newcommand{\eobsp}{\hat{\boldsymbol{e}}_{\mathrm{obs}}^{\perp}}
\newcommand{\erob}{\hat{\boldsymbol{e}}_{\mathrm{rob}}}
\newcommand{\erobp}{\hat{\boldsymbol{e}}_{\mathrm{rob}}^{\perp}}
\newcommand{\vrx}{\tilde v_{\mathrm{rel},x}}
\newcommand{\vry}{\tilde v_{\mathrm{rel},y}}
\newcommand{\nv}{\lVert\vrel\rVert}
\newcommand{\npr}{\lVert\prel\rVert}
\newcommand{\thr}{\tilde\theta}
\newcommand{\tho}{\tilde\theta_{\mathrm{obs}}}
\newcommand{\U}{\mathcal{U}}
\newcommand{\amax}{a_{\max}}
\newcommand{\aobm}{a_{\mathrm{obs,max}}}
\newcommand{\wobm}{\omega_{\mathrm{obs,max}}}
\newcommand{\vobm}{v_{\mathrm{obs,max}}}
\newcommand{\cmin}{c_{\min}}
\newcommand{\Dstar}{\mathcal{D}^{*}}
\newcommand{\Dmin}{\mathcal{D}_{\min}}
\newcommand{\Lmin}{\Lambda_{\min}}
\newcommand{\Lgh}{\boldsymbol{C}'(\boldsymbol{x})^{\!\top}}

\begin{document}

\title{Adversarially Robust Geometric Safety Certificates for Nonholonomic Robots Against Maneuvering  Obstacles

\author{[Author names omitted for blind review]}
\author{Chandan Kumar Sah, Bazeela Banday, and Jishnu Keshavan}
\thanks{The authors are with the Department of Mechanical Engineering, Indian Institute of Science, Bangalore, Karnataka~560012, India 
    ({\tt\footnotesize{email: chandanks@iisc.ac.in, bazeelab@iisc.ac.in, kjishnu@iisc.ac.in}}).
    }
}
\maketitle
\IEEEpeerreviewmaketitle

\begin{abstract}
Safe navigation against obstacles that can actively maneuver within bounded capabilities remains challenging: robust control barrier function methods typically treat obstacle actions as generic disturbances, while differential-game approaches are computationally expensive for online navigation. We propose an adversarially robust geometric certificate that accounts for the worst-case effect of admissible obstacle maneuvers directly in the safe-set geometry through a closed-form contraction of the certificate parameters. The construction exploits a structural property of line-of-sight (LoS) certificates: the robot and obstacle actions enter the certificate through a common state-dependent geometric gain. This gain cancels in the worst-case comparison, reducing the differential game to a direct comparison between obstacle maneuvering capability and the weaker of the robot's longitudinal and steering authorities. Instantiated on the parabolic certificate, the construction yields Adversarially Robust Dynamic Parabolic Control Barrier Functions (AR-DPCBF), for which we establish sufficient conditions for forward invariance of the contracted safe set against all admissible obstacle maneuvers under kinematic bicycle dynamics with bounded inputs. When the obstacle capability is unknown, a sliding-window estimator supplies a high-probability upper bound, allowing the guarantee to be retained with the corresponding coverage probability. We further formulate soft and buffered variants to recover feasibility in dense environments. Simulations across obstacle capabilities, densities, and capability mismatch show substantial reductions in barrier violations and collisions and demonstrate that pointwise robustification of the barrier derivative cannot substitute for contraction of its geometry.
\noindent
\href{https://cks0314.github.io/dummy_page/}{\color{red}[Project page]}
\href{https://github.com/cks0314/Adversarial-DPCBF.git}{\color{red}[Code]\color{black}}

\end{abstract}


%
\section{Introduction}
\label{sec:intro}
 
Safe navigation in dynamic environments requires a robot to remain collision-free despite uncertainty in the future motion of the surrounding obstacles. When obstacles maneuver within bounded acceleration and turning capabilities, safety acquires an adversarial character: the robot must maintain a valid safety certificate against every admissible obstacle action. The difficulty is accentuated for nonholonomic robots, whose limited steering authority couples collision avoidance to the robot--obstacle relative geometry and to the heading dynamics~\cite{haraldsen2024safety, huang2025dynamic}.
 
Control Barrier Functions (CBFs)~\cite{ames2017cbf} provide a computationally efficient way to enforce forward-invariance requirements. For nonholonomic systems, however, control inputs affect position through the vehicle heading, causing distance-based safety constraints to have a higher relative degree. High-Order CBFs~\cite{xiao2019hocbf,xiao2022hocbf} recover explicit input dependence by repeated differentiation, but the resulting constraints can become conservative or jointly infeasible when multiple obstacles are considered. This has motivated approaches that reason directly about relative motion. Velocity Obstacle (VO) methods characterize relative velocities that inevitably lead to collision~\cite{fiorini1998vo}. Collision-cone CBFs incorporate this geometry into safety-critical control~\cite{tayal2024c3bf}. Meanwhile, Dynamic Parabolic CBFs (DPCBFs)~\cite{park2026dpcbf} reduce conservatism using a state-dependent parabolic boundary in the LoS relative-velocity frame. These geometric certificates are defined in terms of relative velocity and clearance and are evaluated using the obstacle's instantaneous velocity.

The premise is restrictive precisely when obstacles maneuver. An accelerating or turning obstacle changes the relative velocity after certification, so the certificate can be invalidated while the constraint the controller enforces remains satisfied and feasible. This failure is silent and it is a loss of soundness rather than of margin.
 
Robust CBF formulations account for bounded disturbances and modeling uncertainty. Worst-case constructions tighten the barrier condition by the largest admissible disturbance effect~\cite{jankovic2018robust,xu2015robustcbf,nguyen2022robust,breeden2023robust}; input-to-state-safe and parameterized formulations trade a quantified set inflation against disturbance magnitude~\cite{kolathaya2019issf,alan2023issf,alan2023paramcbf}; and observer-based schemes reduce conservatism by estimating the disturbance online~\cite{dacs2022robust,dacs2025robust,quan2025observer}. Closest to the present setting, robust barriers of high relative degree have been designed for unknown moving obstacles~\cite{kim2025robust}. Treating obstacle maneuvers as generic disturbances, however, places the margin on the barrier derivative without reflecting how acceleration and turning act on the relative motion, so the resulting margins need not preserve the relative-motion geometry on which the certificate rests.

Differential-game methods provide a principled formulation of this worst-case problem. The set of states from which an obstacle can force a collision can be characterized through Hamilton-Jacobi-Isaacs (HJI) reachability~\cite{mitchell2005hji,bansal2017hj}. However, the computational complexity of HJI methods grows rapidly with state dimension, making repeated online computation difficult for multi-obstacle navigation. The notion of an Inevitable Collision State (ICS)~\cite{fraichard2004ics} provides a corresponding characterization of states that must be excluded under adversarial obstacle behavior. These approaches therefore motivate a tractable approximation of the adversarial safe set that retains explicit worst-case reasoning without requiring online solution of a full differential game.

In this paper, we characterize the adversarial safe set geometrically. Modeling the obstacle as an agent with bounded acceleration and turning rate, we show that its entire set of admissible maneuvers acts on a LoS certificate via a single scalar capability, and that the robot's acceleration and steering act on the same certificate via the same geometric factor. The worst-case comparison, therefore, reduces to an inequality between the maneuvering capability of the obstacle and the actuation authority of the robot.

This reduction also identifies where robustness must be introduced. Rather than robustifying the instantaneous barrier rate, we preserve the worst-case margin in the barrier geometry by contracting its parameters, yielding an inner approximation of the adversarial safe set. We show that pointwise robustification constrains the instantaneous barrier rate, whereas contraction excludes states from which any admissible control cannot maintain the certificate. Instantiating the construction on the parabolic barrier family yields AR-DPCBF, for which we establish forward invariance of the contracted safe set with respect to all admissible obstacle maneuvers for a kinematic bicycle model subject to input constraints.

The framework also addresses uncertainty in the obstacle's capability. We develop an online sliding-window estimator that provides a high-probability upper bound on the true obstacle capability, allowing the estimated capability to be incorporated into the adversarial barrier while retaining the safety guarantee of the known-capability formulation. Since the contracted adversarial safe set can reduce QP feasibility in dense environments, we additionally introduce soft and buffered AR-DPCBF formulations that retain the nominal DPCBF constraint while penalizing violations of the adversarial condition, providing different safety-feasibility trade-offs.
 
The main contributions of the paper are as follows.
 
\begin{enumerate}
\item
We identify a silent failure mode of LoS-frame geometric safety certificates under maneuvering obstacles and formulate the corresponding adversarial CBF condition over the obstacle's structured capability set.

\item
We show that for LoS certificates that are affine in longitudinal and quadratic in lateral relative velocity, the robot and obstacle actions share a common geometric gain. Its cancellation reduces the adversarial condition to a comparison between obstacle-maneuvering capability and robot actuation authority, thereby yielding sufficient conditions for forward invariance.

\item
We introduce a closed-form contraction of the certificate parameters that reserves the worst-case margin in the safe-set geometry. The resulting AR-DPCBF is shown to be strictly stronger than pointwise robustification.

\item
We develop a sliding-window estimator that provides a high-probability upper bound on the unknown obstacle's capability, enabling an adversarial certificate without prior knowledge of its capability.

\item
We formulate hard, soft, and buffered variants to retain feasibility in dense environments, and evaluate them against other baselines using paired Monte Carlo trials.
\end{enumerate}

The remainder of the paper is organized as follows. Section~\ref{sec:problem} introduces the problem formulation and preliminaries. Sections~\ref{sec:params} and ~\ref{sec:validity} develop the proposed AR-DPCBF framework and establish its theoretical guarantees. Section~\ref{sec:estimator} extends the framework to account for unknown obstacle capabilities via online capability estimation, while Section~\ref{sec:controller} presents soft-constrained variants. Section~\ref{sec:results} evaluates the proposed approach through simulation studies, and Section~\ref{sec:conc} concludes the paper.

\section{Background and Problem Statement}
\label{sec:problem}
\subsection{Control Barrier Functions}
\label{subsec:cbf}
Consider the control-affine system
\begin{equation}
  \dot{\boldsymbol{x}} = \boldsymbol{f}(\boldsymbol{x}) + \boldsymbol{g}(\boldsymbol{x})\boldsymbol{u},
  \label{eq:sys}
\end{equation}
where $\boldsymbol{x}\in\mathcal{X}\subset\mathbb{R}^{n}$,
$\boldsymbol{u}\in\mathcal{U}\subset\mathbb{R}^{m}$,
and $\boldsymbol{f},\boldsymbol{g}$ are locally Lipschitz.
Let $h:\mathbb{R}^{n}\to\mathbb{R}$ be continuously
differentiable and define the \emph{safe set}
\begin{equation}
  \label{eq:safe_set}
  \mathcal{C} = \{\boldsymbol{x}\in\mathcal{X} \mid h(\boldsymbol{x})\ge 0\}.
\end{equation}
\begin{definition}[Control Barrier Function~{\cite{ames2017cbf}}]
\label{def:cbf}
A continuously differentiable function
$h:\mathbb{R}^{n}\to\mathbb{R}$ is a
\emph{control barrier function} (CBF) for~\eqref{eq:sys}
if there exists $\psi\in\mathcal{K}_{\infty}$
(a continuous, strictly increasing function with $\psi(0)=0$)
such that
\begin{equation}
  \label{eq:cbf_cond}
  \sup_{\boldsymbol{u}\in\mathcal{U}}
  \bigl[L_{\boldsymbol{f}}h(\boldsymbol{x}) + L_{\boldsymbol{g}}h(\boldsymbol{x})\,\boldsymbol{u}\bigr]
  \ge -\psi(h(\boldsymbol{x}))
  \quad \forall \boldsymbol{x}\in\mathcal{C},
\end{equation}
where $L_{\boldsymbol{f}}h = \nabla h(\boldsymbol{x})^{\top}\boldsymbol{f}(\boldsymbol{x})$ and
$L_{\boldsymbol{g}}h = \nabla h(\boldsymbol{x})^{\top}\boldsymbol{g}(\boldsymbol{x})$ denote Lie derivatives.
\end{definition}
Given a CBF $h$, define the set of safe control inputs
\begin{equation}
  \label{eq:Kcbf}
  \mathcal{K}_{\mathrm{cbf}}(\boldsymbol{x})
  = \bigl\{\boldsymbol{u}\in\mathcal{U}
    \mid L_{\boldsymbol{f}}h(\boldsymbol{x}) + L_{\boldsymbol{g}}h(\boldsymbol{x})\,\boldsymbol{u} \ge -\psi(h(\boldsymbol{x}))\bigr\}.
\end{equation}
Any locally Lipschitz controller $\boldsymbol{u}=\boldsymbol{\pi}(\boldsymbol{x})\in
\mathcal{K}_{\mathrm{cbf}}(\boldsymbol{x})$ renders $\mathcal{C}$
\emph{forward invariant}:   $\boldsymbol{x}(0)\in\mathcal{C}
  \;\Rightarrow\;
  \boldsymbol{x}(t)\in\mathcal{C}
  \quad \text{for all } t\ge 0.$
In practice, the safe control is computed at each time step
by solving the \emph{CBF-QP}
\begin{equation}
\label{eq:cbf_qp}
\begin{aligned}
\boldsymbol{u}^{*}
    &= \arg\min_{\boldsymbol{u}\in\mathcal{U}}
       \|\boldsymbol{u}-\boldsymbol{u}_{\mathrm{ref}}\|^{2} \\
\text{subject to}\quad
    &L_{\boldsymbol{f}}h(\boldsymbol{x})+L_{\boldsymbol{g}}h(\boldsymbol{x})\,\boldsymbol{u}
      \ge -\psi\!\left(h(\boldsymbol{x})\right).
\end{aligned}
\end{equation}
which minimally deviates from a reference control
$\boldsymbol{u}_{\mathrm{ref}}$ while enforcing the CBF constraint.

\subsection{Kinematic Bicycle Model}
\label{subsec:bicycle}

The robot is modeled as a kinematic bicycle~\cite{polack2017kinematic}, a
standard approximation for nonholonomic vehicles at low-to-moderate speeds. The state is $\boldsymbol{x} = [x,y,\theta,v]^{\top}$, comprising the CoM position $(x,y)$, heading $\theta$, and forward speed $v$
(Fig.~\ref{fig:frames}a). Under the small-slip approximations
$\tan\beta\approx\beta$ and $\sin\beta\approx\beta$, the dynamics take the
control-affine form~\eqref{eq:sys},
\begin{equation}
\label{eq:bicycle}
\boldsymbol{f}(\boldsymbol{x})
=
\begin{bmatrix}
v\cos\theta\\
v\sin\theta\\
0\\
0
\end{bmatrix},
\;
\boldsymbol{g}(\boldsymbol{x})
=
\begin{bmatrix}
0 & -v\sin\theta\\
0 &  v\cos\theta\\
0 &  v/\ell_r\\
1 & 0
\end{bmatrix},
\;
\boldsymbol{u}=
\begin{bmatrix}
a\\
\beta
\end{bmatrix},
\end{equation}
where $a$ is the longitudinal acceleration, $\beta$ the slip-angle input,
and $\ell_r$ the rear-axle-to-CoM distance. The admissible control set is
\begin{equation}
\label{eq:U}
\mathcal U=
\left\{
(a,\beta)\;
\middle|\;
|a|\le a_{\max},\;
|\beta|\le\beta_{\max}
\right\},
\end{equation}
where $a_{\max}, \beta_{\max}>0$ denote upper bounds on the magnitudes of the longitudinal acceleration and slip angle, respectively.
It is convenient to separate the two components of the robot velocity. Writing
$\hat{\boldsymbol{e}}_{\mathrm{rob}}=[\cos\theta,\sin\theta]^{\top}$ for the heading direction and
$\hat{\boldsymbol{e}}^{\perp}_{\mathrm{rob}}$ for its counterclockwise normal, \eqref{eq:bicycle} gives $\dot{\boldsymbol{p}}
= v\,\hat{\boldsymbol{e}}_{\mathrm{rob}}
+\;
v\beta\,\hat{\boldsymbol{e}}^{\perp}_{\mathrm{rob}},$
with $\boldsymbol{v}_{\mathrm{rob}}:= v\,\hat{\boldsymbol{e}}_{\mathrm{rob}}$. So the center-of-mass velocity is the sum of the \emph{no-slip} velocity $\boldsymbol{v}_{\mathrm{rob}}$ and a lateral component generated by the slip-angle input.

\begin{figure}[!t]
\centering
\includegraphics[width=\columnwidth]{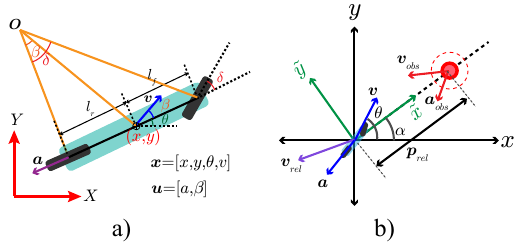}
\caption{Kinematic bicycle model and line-of-sight (LoS) relative-motion geometry. (a) Kinematic bicycle model with state $\boldsymbol{x}=[x,y,\theta,v]^\top$ and control input $\boldsymbol{u}=[a,\beta]^\top$. (b) Robot-obstacle relative parameters expressed in the LoS frame.}
\label{fig:frames}
\end{figure}

\subsection{DPCBF Review}
\label{subsec:dpcbf}

The DPCBF~\cite{park2026dpcbf} instantiates the CBF
condition~\eqref{eq:cbf_cond} for dynamic obstacle avoidance by constructing
a parabolic barrier in the Line-of-Sight (LoS) relative-velocity frame.

\paragraph*{LoS-frame geometry}

Figure~\ref{fig:frames} illustrates the coordinate transformation. The robot and obstacle positions and velocities, $\boldsymbol{p},\boldsymbol{v}_{\mathrm{rob}}$ and
$\boldsymbol{p}_{\mathrm{obs}},\boldsymbol{v}_{\mathrm{obs}}$, respectively, define the relative quantities
\begin{equation}
\label{eq:rel}
\boldsymbol{p}_{\mathrm{rel}}
=
\boldsymbol{p}_{\mathrm{obs}}
-
\boldsymbol{p},
\qquad
\boldsymbol{v}_{\mathrm{rel}}
=
\boldsymbol{v}_{\mathrm{obs}}
-
\boldsymbol{v}_{\mathrm{rob}}.
\end{equation}
The barrier constructed below depends on the obstacle state as well as on the robot state. Accordingly, $\mathcal{X}$ denotes the joint robot-obstacle state space and, with a slight abuse of notation, $\boldsymbol{x}$ denotes the joint state
$(x,y,\theta,v,x_{\mathrm{obs}},y_{\mathrm{obs}},\theta_{\mathrm{obs}},v_{\mathrm{obs}})$ whenever it appears as the argument of $d$, $\lambda$, $\mu$, $h$ or $h^{*}$; the robot dynamics~\eqref{eq:bicycle} act on its first four components and the obstacle dynamics~\eqref{eq:obs_dynamics} on the remaining four. Lie derivatives are taken with respect to~\eqref{eq:bicycle}, the obstacle inputs entering as the disturbance of Definition~\ref{def:acbf}.

The LoS frame (Fig.~\ref{fig:frames}b) is obtained by rotating the global
frame so that its $\tilde{x}$-axis aligns with
$\boldsymbol{p}_{\mathrm{rel}}$. The corresponding LoS angle is, $\alpha = \operatorname{atan2} (p_{\mathrm{rel},y},p_{\mathrm{rel},x}),$ with a rotation matrix

\begin{equation}
\label{eq:R}
\boldsymbol{R}(\alpha)=
\begin{bmatrix}
\cos\alpha & -\sin\alpha\\
\sin\alpha & \cos\alpha
\end{bmatrix}.
\end{equation}

The relative velocity in the LoS-frame is $\tilde{\boldsymbol{v}}_{\mathrm{rel}} = \boldsymbol{R}(-\alpha) \boldsymbol{v}_{\mathrm{rel}} $, with components $(\tilde v_{\mathrm{rel}, x}, \tilde v_{\mathrm{rel}, y})$. The tangential clearance from the combined safety circle of radius $r = r_{\mathrm{rob}} + r_{\mathrm{obs}}$ is

\begin{equation}
\label{eq:clearance}
d(\boldsymbol{x})=
\sqrt{\|\boldsymbol{p}_{\mathrm{rel}}\|^2-r^2}.
\end{equation}

\paragraph*{DPCBF candidate}

The DPCBF barrier function~\cite{park2026dpcbf} is
\begin{equation}
\label{eq:dpcbf_review}
h(\boldsymbol{x})=
\tilde v_{\mathrm{rel},x}
+\lambda(\boldsymbol{x})\tilde v_{\mathrm{rel},y}^2
+\mu(\boldsymbol{x}),
\end{equation}
with state-dependent parameters
\begin{equation}
\label{eq:dpcbf_params}
\lambda(\boldsymbol{x})
=
k_\lambda
\frac{d(\boldsymbol{x})}
{\|\boldsymbol{v}_{\mathrm{rel}}\|},
\qquad
\mu(\boldsymbol{x})
=
k_\mu d(\boldsymbol{x}),
\end{equation}
for tunable gains $k_\lambda,k_\mu>0$. The zero level set
$\{h=0\}$ defines a parabola in the LoS-frame velocity space: states with $h(\boldsymbol{x})\ge0$ are safe, whereas those with $h(\boldsymbol{x})<0$ are unsafe.

\subsection{Obstacle Capability Set and Adversarial CBF}
\label{subsec:capability_acbf}

We model the obstacle as a controlled agent whose Maneuvering  authority is bounded by a \emph{capability set}, defined as,

\begin{definition}[Obstacle Capability Set]
\label{def:F}
We model each obstacle as a kinematic agent whose motion is governed by unicycle-like dynamics, with longitudinal acceleration $a_{\mathrm{obs}}$ and angular rate $\omega_{\mathrm{obs}}$ as inputs.
The obstacle state $(x_{\mathrm{obs}},y_{\mathrm{obs}},\theta_{\mathrm{obs}},
v_{\mathrm{obs}})$ evolves as
\begin{equation}
  \label{eq:obs_dynamics}
  \begin{aligned}
    \dot{x}_{\mathrm{obs}}      &= v_{\mathrm{obs}}\cos\theta_{\mathrm{obs}}, &
    \dot{y}_{\mathrm{obs}}      &= v_{\mathrm{obs}}\sin\theta_{\mathrm{obs}},\\
    \dot{\theta}_{\mathrm{obs}} &= \omega_{\mathrm{obs}},                      &
    \dot{v}_{\mathrm{obs}}      &= a_{\mathrm{obs}}.
  \end{aligned}
\end{equation}
The obstacle's Maneuvering  authority is bounded by the
\emph{capability set}
\begin{equation}
\label{eq:F}
\mathcal{F}
=
\left\{
(a_{\mathrm{obs}},\omega_{\mathrm{obs}})\in\mathbb{R}^{2}
\;\middle|\;
\begin{aligned}
|a_{\mathrm{obs}}| &\le a_{\mathrm{obs,max}},\\
|\omega_{\mathrm{obs}}| &\le \omega_{\mathrm{obs,max}}
\end{aligned}
\right\}.
\end{equation}
where $a_{\mathrm{obs,max}}, \omega_{\mathrm{obs,max}}\ge0$ are known upper bounds on the obstacle's translational acceleration and turning rate, respectively. The obstacle speed
$v_{\mathrm{obs}}(t)\in[0,v_{\mathrm{obs,max}}]$ for all $t$, where $v_{\mathrm{obs,max}}$ is a physical speed limit.
\end{definition}

\begin{remark}
\label{rem:F_shape}
The rectangular shape of $\mathcal{F}$ decouples the two authority axes and is the key structural property that allows the supremum in Section~\ref{sec:validity} to separate into independent bounds on $|P(\phi)|$ and $|Q(\phi)|$, yielding closed-form adversarial parameters. The extension to a more physically realistic Kamm circle constraint $a_{\mathrm{obs}}^2 + (v_{\mathrm{obs}} \omega_{\mathrm{obs}})^2 \le \kappa_{\mathrm{Kamm}}^2$ would require solving a trust-region problem.
\end{remark}

\begin{definition}[Inevitable Collision State~{\cite{fraichard2004ics}}]
\label{def:ICS}
A state $\boldsymbol{x}$ is an \emph{inevitable collision state} (ICS) if, for every robot control $\boldsymbol{u}_{\mathrm{rob}}:[0,\infty)\to\mathcal{U}$, there exists a disturbance trajectory $(a_{\mathrm{obs}}(t), \omega_{\mathrm{obs}}(t))\in\mathcal{F}$ such that $\|\boldsymbol{p}(t)-\boldsymbol{p}_{\mathrm{obs}} (t)\|<r$ for some $t>0$.
The set of all ICS is $\mathcal{I}\subset\mathcal{X}$.
\end{definition}

\begin{remark}
\label{rem:ICS_openloop}
In Definition~\ref{def:ICS}, the obstacle may choose its maneuver knowing the robot's entire control signal. This is the \emph{open-loop} form of the ICS. It gives the obstacle at least as much power as the nonanticipative strategies used in differential games~\cite{mitchell2005hji,choi2021robust}, since every such strategy yields an admissible disturbance trajectory once the robot control is fixed.
\end{remark}

We seek a barrier function whose zero super-level set excludes the
ICS region. To make this precise, let $h^{*}:\mathcal{X}\to\mathbb{R}$ be a function to be designed and define: 
$  \mathcal{C}^{*} \;:=\; \{\boldsymbol{x}\in\mathcal{X} \mid h^{*}(\boldsymbol{x})\ge 0\}.$ 
Exactly characterizing $\mathcal{I}$ requires solving a HJI PDE, which is computationally intractable in general. In particular, choosing $h^{*}$ so that
$\mathcal{C}^{*}=\mathcal{X}\setminus\mathcal{I}$ exactly is not tractable. Our approach instead constructs $\mathcal{C}^{*}$ as a \emph{conservative inner approximation}, $\mathcal{C}^{*} \subseteq\mathcal{X} \setminus\mathcal{I}$, via a parabolic ansatz whose parameters are chosen to satisfy Nagumo's condition against all $\mathcal{F}$-admissible obstacle maneuvers.

\begin{definition}[Adversarial CBF (A-CBF)]
\label{def:acbf}
Let $h^{*}:\mathcal{X}\to\mathbb{R}$ be continuously
differentiable with $\nabla h^{*}(\boldsymbol{x})\neq 0$ on
$\partial\mathcal{C}^{*}=\{h^{*}=0\}$, and let
$\mathcal{B}\subseteq\partial\mathcal{C}^{*}$. Then $h^{*}$ is an \emph{adversarial control barrier function} for~\eqref{eq:bicycle} with respect to $\mathcal{F}$ on $\mathcal{B}$ if, for all $\boldsymbol{x}\in\mathcal{B}$ and all
$(a_{\mathrm{obs}},\omega_{\mathrm{obs}})\in\mathcal{F}$, there exists $\boldsymbol{u}\in\mathcal{U}$ satisfying
\begin{equation}
  \label{eq:acbf}
  \dot h^{*}(\boldsymbol{x},\boldsymbol{u},a_{\mathrm{obs}},\omega_{\mathrm{obs}})
  \;\ge\; 0,
\end{equation}
where $\dot h^{*}$ is evaluated along the robot dynamics~\eqref{eq:bicycle} and the obstacle dynamics~\eqref{eq:obs_dynamics}, and is therefore a function of the robot input and of the obstacle input alike.
\end{definition}

Condition~\eqref{eq:acbf} is the boundary (Nagumo) form of the standard CBF condition~\eqref{eq:cbf_cond}, and is equivalent to requiring $\dot h^{*}\ge0$ at $\boldsymbol{x}$ for every admissible obstacle input and some admissible robot input. The implemented controller enforces the relaxed inequality $\dot h^{*}\ge-\psi(h^{*})$, which coincides with \eqref{eq:acbf} on $\partial\mathcal C^{*}$.

\begin{remark}
The standard CBF condition~\eqref{eq:cbf_cond} is a single-agent maximization,
$
\sup_{\boldsymbol{u}\in\mathcal{U}}
\left(
L_{\boldsymbol{f}} h + L_{\boldsymbol{g}} h\,\boldsymbol{u}
\right)
\geq
-\psi(h),
$
where the robot is the only decision maker and the obstacle dynamics is fixed. In contrast, the A-CBF condition~\eqref{eq:acbf} introduces a universal quantification over the obstacle capability set,
$
\forall\,(a_{\rm obs},\omega_{\rm obs})\in\mathcal{F}, \;
\exists\,\boldsymbol{u}\in\mathcal{U},
$
such that the barrier inequality holds. Thus, the safety certificate is required to hold for every admissible obstacle maneuver rather than a single nominal obstacle model. This yields the minimax logical structure characteristic of differential games.
\end{remark}



\subsection{LoS-Frame Relative Dynamics}
\label{subsec:LoS_reducibility}

A key step in the adversarial formulation is transferring the robot-obstacle interaction to the LoS frame. The following lemma derives the LoS-frame relative dynamics exactly and isolates the frame-rotation terms.

\begin{lemma}\label{lem:los}
Under the robot model~\eqref{eq:bicycle} and obstacle dynamics~\eqref{eq:obs_dynamics}, the LoS-frame relative velocity evolves \emph{exactly} as
\begin{align}
\dot{\tilde{v}}_{\mathrm{rel},x}
  &= a_{\mathrm{obs}}\cos\tilde{\theta}_{\mathrm{obs}}
   - v_{\mathrm{obs}}\omega_{\mathrm{obs}}\sin\tilde{\theta}_{\mathrm{obs}}
   \nonumber\\
  &\quad - a\cos\tilde{\theta}
   + \frac{v^{2}}{\ell_r}\,\beta\sin\tilde{\theta}
   + R_x, \label{eq:losx}\\
\dot{\tilde{v}}_{\mathrm{rel},y}
  &= a_{\mathrm{obs}}\sin\tilde{\theta}_{\mathrm{obs}}
   + v_{\mathrm{obs}}\omega_{\mathrm{obs}}\cos\tilde{\theta}_{\mathrm{obs}}
   \nonumber\\
  &\quad - a\sin\tilde{\theta}
   - \frac{v^{2}}{\ell_r}\,\beta\cos\tilde{\theta}
   + R_y, \label{eq:losy}
\end{align}
where $\tilde{\theta} = \theta - \alpha$ and $\tilde{\theta}_{\mathrm{obs}} = \theta_{\mathrm{obs}} - \alpha$ are the
robot and obstacle headings in the LoS frame, respectively, and the frame-rotation terms admit the exact closed form: 
\begin{equation}\label{eq:rot-exact}
R_x = \dot{\alpha}\,\tilde{v}_{\mathrm{rel},y}, \qquad
R_y = -\dot{\alpha}\,\tilde{v}_{\mathrm{rel},x},
\end{equation}
with
\begin{equation}\label{eq:alphadot-split}
\dot{\alpha}
 = \frac{\boldsymbol{p}_{\mathrm{rel}} \times \dot{\boldsymbol{p}}_{\mathrm{rel}}}
        {\|\boldsymbol{p}_{\mathrm{rel}}\|^{2}}
 = \underbrace{\frac{\boldsymbol{p}_{\mathrm{rel}} \times \boldsymbol{v}_{\mathrm{rel}}}
        {\|\boldsymbol{p}_{\mathrm{rel}}\|^{2}}}_{\dot{\alpha}_0}
 \;\underbrace{-\;\frac{v\beta\,
        (\boldsymbol{p}_{\mathrm{rel}} \times \hat{\boldsymbol{e}}^{\perp}_{\mathrm{rob}})}
        {\|\boldsymbol{p}_{\mathrm{rel}}\|^{2}}}_{\dot{\alpha}_{\mathrm{slip}}},
\end{equation}
where $\dot{\boldsymbol{p}}_{\mathrm{rel}}
 = \boldsymbol{v}_{\mathrm{rel}} - v\beta\,\hat{\boldsymbol{e}}^{\perp}_{\mathrm{rob}}$
follows from~\eqref{eq:bicycle}. Consequently,
\begin{equation}\label{eq:rot-bound}
|R_x|,\,|R_y|
 \;\le\; \frac{\|\boldsymbol{v}_{\mathrm{rel}}\|
   \bigl(\|\boldsymbol{v}_{\mathrm{rel}}\| + v\,|\beta|\bigr)}
   {\|\boldsymbol{p}_{\mathrm{rel}}\|}
 \;\le\; \frac{C_R}{\|\boldsymbol{p}_{\mathrm{rel}}\|},
\end{equation}
with
$C_R := v_{\mathrm{rel,max}}\bigl(v_{\mathrm{rel,max}}
        + v_{\max}\beta_{\max}\bigr)$,
and the rotation terms vanish as $\|\boldsymbol{p}_{\mathrm{rel}}\| \to \infty$.
\end{lemma}

\begin{proof}
\emph{Frame rotation.}
Differentiating
$\tilde{\boldsymbol{v}}_{\mathrm{rel}} = \boldsymbol{R}(-\alpha)\boldsymbol{v}_{\mathrm{rel}}$ gives
$\dot{\tilde{\boldsymbol{v}}}_{\mathrm{rel}}
 = \boldsymbol{R}(-\alpha)\dot{\boldsymbol{v}}_{\mathrm{rel}}
 + \dot{\boldsymbol{R}}(-\alpha)\boldsymbol{v}_{\mathrm{rel}}$.
Since $\tfrac{\mathrm{d}}{\mathrm{d}\varphi}\boldsymbol{R}(\varphi) = \boldsymbol{J}\,\boldsymbol{R}(\varphi)$ with
$\boldsymbol{J} = \begin{bsmallmatrix} 0 & -1 \\ 1 & 0 \end{bsmallmatrix}$
the $90^{\circ}$ counterclockwise rotation generator,
$\dot{\boldsymbol{R}}(-\alpha) = -\dot{\alpha}\,\boldsymbol{J}\,\boldsymbol{R}(-\alpha)$, and hence,
$\dot{\boldsymbol{R}}(-\alpha)\boldsymbol{v}_{\mathrm{rel}}
 = -\dot{\alpha}\,\boldsymbol{J}\,\tilde{\boldsymbol{v}}_{\mathrm{rel}}
 = \dot{\alpha}
   \begin{bmatrix} \tilde{v}_{\mathrm{rel},y} &
                   -\tilde{v}_{\mathrm{rel},x} \end{bmatrix}^{\top},$
which is~\eqref{eq:rot-exact}. Since
$\alpha = \operatorname{atan2}(p_{\mathrm{rel},y}, p_{\mathrm{rel},x})$,
$\dot{\alpha} = (\boldsymbol{p}_{\mathrm{rel}} \times \dot{\boldsymbol{p}}_{\mathrm{rel}})
/\|\boldsymbol{p}_{\mathrm{rel}}\|^{2}$. Substituting
$\dot{\boldsymbol{p}}_{\mathrm{rel}} = \boldsymbol{v}_{\mathrm{rel}}
 - v\beta\,\hat{\boldsymbol{e}}^{\perp}_{\mathrm{rob}}$
yields~\eqref{eq:alphadot-split}.

\emph{Translational accelerations.}
The obstacle velocity is
$\boldsymbol{v}_{\mathrm{obs}} = v_{\mathrm{obs}}\hat{\boldsymbol{e}}_{\mathrm{obs}}$ with
$\dot{\hat{\boldsymbol{e}}}_{\mathrm{obs}}
 = \dot{\theta}_{\mathrm{obs}}\hat{\boldsymbol{e}}^{\perp}_{\mathrm{obs}}$, so
by~\eqref{eq:obs_dynamics},
$\dot{\boldsymbol{v}}_{\mathrm{obs}}
 = a_{\mathrm{obs}}\hat{\boldsymbol{e}}_{\mathrm{obs}}
 + v_{\mathrm{obs}}\omega_{\mathrm{obs}}\hat{\boldsymbol{e}}^{\perp}_{\mathrm{obs}}$.
Similarly, the no-slip robot velocity
$\boldsymbol{v}_{\mathrm{rob}} = v\,\hat{\boldsymbol{e}}_{\mathrm{rob}}$ satisfies, $\dot{\boldsymbol{v}}_{\mathrm{rob}}
 = \dot{v}\,\hat{\boldsymbol{e}}_{\mathrm{rob}}
 + v\,\dot{\theta}\,\hat{\boldsymbol{e}}^{\perp}_{\mathrm{rob}}
 = a\,\hat{\boldsymbol{e}}_{\mathrm{rob}}
 + \frac{v^{2}}{\ell_r}\,\beta\,\hat{\boldsymbol{e}}^{\perp}_{\mathrm{rob}},$ using $\dot{v} = a$ and $\dot{\theta} = (v/\ell_r)\beta$
from~\eqref{eq:bicycle}. Using $\boldsymbol{R}(-\alpha)$,
\begin{equation*}
\boldsymbol{R}(-\alpha)\hat{\boldsymbol{e}}_{\mathrm{obs}}
 = \begin{bmatrix} \cos\tilde{\theta}_{\mathrm{obs}} \\
                   \sin\tilde{\theta}_{\mathrm{obs}} \end{bmatrix},
\quad
\boldsymbol{R}(-\alpha)\hat{\boldsymbol{e}}^{\perp}_{\mathrm{obs}}
 = \begin{bmatrix} -\sin\tilde{\theta}_{\mathrm{obs}} \\
                   \phantom{-}\cos\tilde{\theta}_{\mathrm{obs}} \end{bmatrix},
\end{equation*}
and analogously for
$\hat{\boldsymbol{e}}_{\mathrm{rob}}, \hat{\boldsymbol{e}}^{\perp}_{\mathrm{rob}}$ with
$\tilde{\theta}$. Substituting
$\dot{\boldsymbol{v}}_{\mathrm{rel}}
 = \dot{\boldsymbol{v}}_{\mathrm{obs}} - \dot{\boldsymbol{v}}_{\mathrm{rob}}$ and separating
the $\tilde{x}$- and $\tilde{y}$-components
yields~\eqref{eq:losx}--\eqref{eq:losy}.

\emph{Residual bound.}
From~\eqref{eq:rot-exact},
$|R_x|, |R_y| \le |\dot{\alpha}|\,\|\boldsymbol{v}_{\mathrm{rel}}\|$, and
from~\eqref{eq:alphadot-split},
$|\dot{\alpha}| \le \|\dot{\boldsymbol{p}}_{\mathrm{rel}}\|
/\|\boldsymbol{p}_{\mathrm{rel}}\|
\le (\|\boldsymbol{v}_{\mathrm{rel}}\| + v|\beta|)/\|\boldsymbol{p}_{\mathrm{rel}}\|$,
giving~\eqref{eq:rot-bound}.
\end{proof}


\section{Adversarially Robust Barrier Parameters}
\label{sec:params}

This section constructs the adversarial barrier $h^{*}$ and establishes that its safe set is a subset of the nominal DPCBF safe set, thus inheriting the collision-avoidance guarantee of~\cite{park2026dpcbf}.

\subsection{Setup and the maneuver gap}

We build on the DPCBF formulation. Under the no-slip assumption ($\beta=0$), differentiating the clearance $d(\boldsymbol{x})=\sqrt{\|\prel\|^{2}-r^{2}}$ along the robot dynamics gives
\begin{equation}
\dot d_{0}
= {\prel^{\!\top}\vrel}/{d}
= {\|\prel\|\,\vrx}/{d},
\label{eq:ddot}
\end{equation}
so the clearance rate and longitudinal LoS velocity $\vrx$ have the same sign. To account for maneuvering obstacles, we contract the nominal DPCBF barrier~\eqref{eq:dpcbf_review} to reserve a margin against worst-case admissible maneuvers. The following subsections derive the contraction from obstacle capability and establish the resulting safety guarantee.

\subsection{Standing assumptions}

\begin{assumption}\label{as:obs}
The obstacle inputs lie in the rectangle $\F=\{(a_{\mathrm{obs}}, \omega_{\mathrm{obs}}):|a_{\mathrm{obs}}|\le\aobm,\,
|\omega_{\mathrm{obs}}|\le\wobm\}$, and, per
Definition~\ref{def:F}, $v_{\mathrm{obs}}\le\vobm$. Since
$\dot{\boldsymbol{v}}_{\mathrm{obs}}=a_{\mathrm{obs}}\eobs+v_{\mathrm{obs}}\omega_{\mathrm{obs}}\eobsp$, the
scalar, obstacle capability,
\begin{equation}
\kappa:=\aobm+\vobm\wobm\quad[\mathrm{m\,s^{-2}}]
\label{eq:kappa}
\end{equation}
bounds $\lVert\dot{\boldsymbol{v}}_{\mathrm{obs}}\rVert$ over $\F$. It is the single number summarizing how hard the obstacle can maneuver.
\end{assumption}

\begin{assumption}\label{as:rob}
The robot holds a forward speed $v\in[v_{\min},v_{\max}]$ with $v_{\min}>0$, and its inputs $\boldsymbol{u}=(a,\beta)$ obey $|a|\le\amax$, $|\beta|\le\beta_{\max}$, i.e.\ $\boldsymbol{u}\in\U$. The lower bound $v_{\min}>0$ is essential because the bicycle's steering authority scales as $v^{2}/\ell_{r}$ and vanishes at rest, so a stationary robot cannot steer to defend the barrier.
\end{assumption}

\begin{assumption}
\label{as:geo}
Let $v_{\mathrm{rel,min}}>0$ and $d_{\max}>r$ be design parameters and define the \emph{engagement envelope} \begin{equation}
\Omega:=\bigl\{\boldsymbol{x}\in\mathcal{X}\;:\; d(\boldsymbol{x})\le d_{\max},\;\;\nv\ge v_{\mathrm{rel,min}}\bigr\}.
\label{eq:envelope} 
\end{equation}
The parameters $v_{\mathrm{rel,min}}$ and $d_{\max}$ are design parameters specified in Corollary~\ref{cor:floor} and Lemma~\ref{lem:Lmin} (Appendix~\ref{app:proofs}).
\end{assumption}

\begin{assumption}[DPCBF clearance property~\cite{park2026dpcbf}]\label{as:clear}
For every $t\ge0$ and every closed-loop trajectory with $d(0)>r$ satisfying $h(\boldsymbol{x}(\tau))\ge0$ for all $\tau\in[0,t]$, the clearance satisfies $d(\tau)\ge r$ for all $\tau\in[0,t]$.
\end{assumption}

\begin{proposition}
\label{prop:bounds}
Under Assumptions~\ref{as:obs}-\ref{as:geo}:
$v_{\mathrm{rel,min}}
\le\nv
\le v_{\mathrm{rel,max}},
\;
\sqrt{2}\,r
\le\npr
\le p_{\max},$
throughout, $\mathcal{B}:=\partial\Cset\cap\Omega\cap\{d\ge r\},$ where, 
$v_{\mathrm{rel,max}}:=v_{\max}+\vobm,
\; p_{\max}:=\sqrt{d_{\max}^2+r^2}.$
Consequently, every denominator appearing below is uniformly bounded
away from zero on $\mathcal{B}$, and $\mathcal{B}$ is compact.
\end{proposition}

\begin{proof}
The bounds follow directly from Assumptions~\ref{as:obs}-\ref{as:geo}, the triangle inequality, and the identity $\npr^2=d^2+r^2$. Compactness of $\mathcal{B}$ is immediate from its definition as a closed subset of a compact set.
\end{proof}

\begin{remark}
\label{rem:envelope}
The engagement envelope $\Omega$ specifies the domain over which the adversarial certificate is established. It is not enforced as a state constraint by the QP. If $\nv<v_{\mathrm{rel,min}}$, the parameter $\lambda=k_\lambda d/\nv$ becomes ill-conditioned, and the certificate we establish later (Theorem~\ref{thm:valid}) no longer applies, as in the nominal DPCBF formulation~\cite{park2026dpcbf}.
\end{remark}

\begin{figure*}[t]
  \centering
  \includegraphics[width=0.95\textwidth]{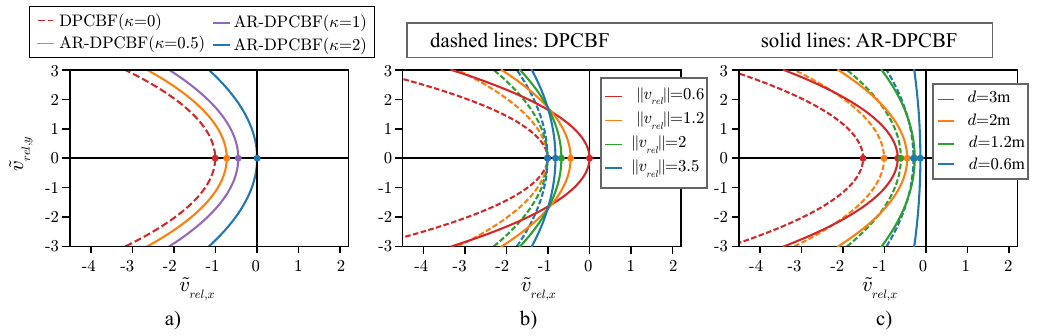}
\caption{Adversarial contraction of the safe-set boundary in the line-of-sight velocity frame. Dashed: DPCBF $\{h=0\}$; solid: AR-DPCBF $\{h^{*}=0\}$. The safe set lies to the right. Since $\lambda^{*}\leq\lambda$ and $\mu^{*}\leq\mu$, $\{h^{*}\geq0\}\subseteq\{h\geq0\}$. (a) Increasing $\kappa$ increases contraction, recovering DPCBF at $\kappa=0$. (b) Lower $\|\boldsymbol{v}_{\mathrm{rel}}\|$ increases contraction through the $1/\|\boldsymbol{v}_{\mathrm{rel}}\|$ reaction-time scaling. (c) Contraction increases with clearance $d$.}
  \label{fig:variation}
\end{figure*}

\subsection{The contracted parameters}
Motivated by the maneuver gap identified above, we derive contracted barrier parameters that allocate part of the nominal DPCBF certificate as a robustness margin against worst-case admissible obstacle maneuvers.

\begin{theorem}\label{thm:params}
Let $\gamma>0$ be a dimensionless design gain and $\amax$ the reference
acceleration of Assumption~\ref{as:rob}. Define
\begin{align}
\lambda^{*}(\boldsymbol{x})&=\frac{k_{\lambda}d}{\nv}-\Delta\lambda,
& \Delta\lambda&=\frac{\kappa}{\gamma\,\amax\,\nv}\ \ge0,
\label{eq:lamstar}\\
\mu^{*}(\boldsymbol{x})&=k_{\mu}d-\Delta\mu,
& \Delta\mu&=\frac{\kappa\,d}{\gamma\,\nv}\ \ge0,
\label{eq:mustar}
\end{align}
and the adversarial barrier,
\begin{eqnarray}
\label{eq:h_star}
h^{*}(\boldsymbol{x})=\vrx+\lambda^{*}\vry^{2}+\mu^{*}.
\end{eqnarray}
Writing the discarded margin as the \emph{buffer}
\begin{equation}
b(\boldsymbol{x}):=\Delta\lambda\,\vry^{2}+\Delta\mu\ \ge0,\qquad h^{*}=h-b,
\label{eq:buffer}
\end{equation}
the construction has three properties:
\begin{enumerate}
\item[(i)] \emph{Contraction.} $\lambda^{*}\le\lambda$ and $\mu^{*}\le\mu$, with $b$ increasing in $\kappa$; thus the unsafe set $\{h^{*}<0\}=\{h<b\}$ grows as the obstacle becomes more capable.
\item[(ii)] \emph{Recovery.} At $\kappa=0$, $b\equiv0$ and $h^{*}\equiv h$: the nominal DPCBF is recovered exactly.
\item[(iii)] \emph{Validity floor.} $\lambda^{*}\ge0$ and $\mu^{*}\ge0$ hold iff $\kappa\le\gamma\amax k_{\lambda}d$ and $\kappa\le\gamma k_{\mu}\nv$. Past these the parabola would invert and no certificate of this form exists at $(d,\nv)$.
\end{enumerate}
\end{theorem}

\begin{proof}
All three claims follow directly from \eqref{eq:lamstar}-\eqref{eq:mustar}. Since $\kappa, d,\gamma,\amax,\nv$ are all positive, $\Delta\lambda$ and $\Delta\mu$ are non-negative and strictly increasing in $\kappa$. Subtracting them gives $\lambda^{*}\le\lambda$, $\mu^{*}\le\mu$, which is (i). Substituting into $h^{*}=\vrx+\lambda^{*}\vry^{2}+\mu^{*}$ and comparing with \eqref{eq:dpcbf_review} yields $h^{*}=h-(\Delta\lambda\,\vry^{2}+\Delta\mu)=h-b$, and $b\ge0$ because both increments are; hence $\{h^{*}<0\}=\{h<b\}$, which contains $\{h<0\}$ and expands
as $b$ (equivalently $\kappa$) grows. Setting $\kappa=0$ makes both increments vanish, so $b\equiv0$ and $h^{*}\equiv h$, proving (ii). For (iii), $\lambda^{*}\ge0\Leftrightarrow k_{\lambda}d/\nv\ge\kappa/(\gamma\amax\nv) \Leftrightarrow\kappa\le\gamma\amax k_{\lambda}d$, and likewise $\mu^{*}\ge0\Leftrightarrow\kappa\le\gamma k_{\mu}\nv$.
\end{proof}

\begin{corollary}
\label{cor:floor}
On $\Omega\cap\{d\ge r\}$ we have $d\ge r$ and $\nv\ge v_{\mathrm{rel,min}}$. Hence, by Theorem~\ref{thm:params}(iii), the floor $\lambda^{*},\mu^{*}\ge0$ holds throughout $\Omega\cap\{d\ge r\}$ if and only if
\begin{equation}
\kappa\le\gamma\,\amax k_{\lambda}\,r,
\qquad
\kappa\le\gamma k_{\mu}v_{\mathrm{rel,min}}.
\label{eq:floor}
\end{equation}
\end{corollary}

Conditions~\eqref{eq:floor} ensure that the reserved margin does not exhaust the barrier.

Fig.~\ref{fig:variation} visualizes the contraction of the AR-DPCBF safe set in the LoS velocity frame. The adversarial boundary $\{h_j^*=0\}$ lies to the right of the nominal DPCBF boundary $\{h_j=0\}$, yielding the contracted safe set $\{h_j^*\ge0\}\subseteq\{h_j\ge0\}$. The three panels illustrate the dependence predicted by Theorem~\ref{thm:params}: the contraction increases with obstacle capability $\kappa$ (a), decreases with increasing relative speed $\|\vrel\|$ owing to the $1/\|\vrel\|$ reaction-time scaling (b), and becomes more pronounced as the clearance $d$ decreases (c). When $\kappa=0$, the AR-DPCBF boundary coincides with the nominal DPCBF boundary, recovering the original certificate exactly.

\subsection{Sizing the buffer}
The contractions in \eqref{eq:lamstar}-\eqref{eq:mustar} follow from bounding the effect of admissible obstacle maneuvers. Consider a head-on encounter ($\vry=0$), for which the nominal DPCBF permits a maximum closing speed of $\mu$. The corresponding time-to-contact is $\tau_c=d/\nv$. By Assumption~\ref{as:obs}, the obstacle can change its velocity by at most $\kappa\tau_c$ over this interval. Conservatively assigning this change to the closing-speed component gives 
$    \Delta\mu
    = \frac{\kappa\tau_c}{\gamma}
    = \frac{\kappa d}{\gamma\nv},$
where $\gamma\ge1$ controls the conservatism of the contraction. The contraction $\Delta\lambda$ is derived analogously, preserving the same $O(1/\nv)$ reaction-time scaling while matching the DPCBF curvature to the maneuver-inflated collision cone of \cite[\S\,III]{park2026dpcbf}.

\subsection{Safety inheritance}
The role of this margin is structural: it ensures that invariance of $h^*$ implies invariance of $h$, thereby transferring the safety guarantee.

\begin{proposition}\label{prop:safe}
With $b\ge0$ from \eqref{eq:buffer}, $\{h^{*}\ge0\}\subseteq\{h\ge0\}$. Hence if
the robot keeps $h^{*}(\boldsymbol{x}(t))\ge0$ for all $t\in[0,t_{1}]$, then under
Assumption~\ref{as:clear} the clearance satisfies $d(t)\ge r$.
\end{proposition}

\begin{proof}
Since $b\ge0$, we have $h^{*}=h-b\le h$ pointwise, so $h^{*}(\boldsymbol{x})\ge0$ implies $h(\boldsymbol{x})\ge0$, i.e., $\{h^{*}\ge0\}\subseteq\{h\ge0\}$. Now suppose the closed loop maintains $h^{*}(\boldsymbol{x}(t))\ge0$ on $[0,t_{1}]$. By the inclusion, the same trajectory satisfies $h(\boldsymbol{x}(t))\ge0$ on $[0,t_{1}]$. Assumption~\ref{as:clear} applied to this trajectory then yields $d(t)\ge r$ on $[0,t_{1}]$.
\end{proof}

\section{Validity of the Adversarial Certificate}
\label{sec:validity}

We prove that $h^{*}$ satisfies the adversarial CBF condition for the bicycle model~\eqref{eq:bicycle} throughout the certified engagement envelope $\Omega$. Consequently, $\Cset$ is forward invariant over $\Omega$ by a standard first-violation argument, and Proposition~\ref{prop:safe} guarantees collision avoidance.

\subsection{The exact derivative}
\label{ssec:exact}

Differentiating $h^{*}$~\eqref{eq:h_star} along the dynamics gives,
\begin{equation}
\dot h^{*}=\underbrace{\dot{\tilde v}_{\mathrm{rel},x}
+2\lambda^{*}\vry\,\dot{\tilde v}_{\mathrm{rel},y}}_{\text{velocity channel}}
+\underbrace{\dot\lambda^{*}\vry^{2}+\dot\mu^{*}}_{\text{parameter channel}}.
\label{eq:hdot-raw}
\end{equation}
The velocity channel is governed by the exact LoS relative-velocity dynamics of Lemma~\ref{lem:los}, whereas the parameter channel describes the evolution of the state-dependent barrier parameters. Note that the parameters $\lambda^{*}$ and $\mu^{*}$ depend on the state only through the clearance $d$ and the relative speed $\nv$.

\begin{lemma}
\label{lem:sigma}
Let
$\dot{\boldsymbol{v}}_{\mathrm{rob|ctrl}}
=
a\,\erob
+
\frac{v^{2}}{\ell_r}\beta\,\erobp$
denote the input-generated acceleration of the no-slip robot velocity. Then the parameter channel of~\eqref{eq:hdot-raw} admits the exact decomposition
\begin{equation}
\dot\lambda^{*}\vry^{2}+\dot\mu^{*}
=
A(\boldsymbol{x})\,\dot d
+
B(\boldsymbol{x})\, \dot{\nv},
\label{eq:param-AB}
\end{equation}
where
\begin{align}
A(\boldsymbol{x})
&{=}
\vry^{2}\partial_d\lambda^{*}
{+}
\partial_d\mu^{*}
{=}
\frac{k_\lambda\vry^{2}}{\nv}
{+}
\left(k_\mu{-}\frac{\kappa}{\gamma\nv}\right),
\label{eq:Adef}
\\
B(\boldsymbol{x})
&=
\vry^{2}\partial_{\nv}\lambda^{*}
+
\partial_{\nv}\mu^{*}
=
-\sigma(\boldsymbol{x})\,\nv,
\label{eq:Bdef}
\end{align}
\begin{equation}
\mathrm{where,} \; \sigma(\boldsymbol{x})
=
\frac{\lambda^{*}\vry^{2}}{\nv^{2}}
-
\frac{\kappa d}{\gamma\nv^{3}}.
\label{eq:sigma}
\end{equation}
\begin{align}
\mathrm{Also,} \; B(\boldsymbol{x})\,\dot{\nv}
&{=}
{-}\sigma(\boldsymbol{x})\,
\vrel\!\cdot\!
\bigl(
\dot{\boldsymbol{v}}_{\mathrm{obs}}
{-}
\dot{\boldsymbol{v}}_{\mathrm{rob|ctrl}}
\bigr),
\label{eq:sigma-split}
\\
\dot d
&=
\frac{\prel\!\cdot\!\vrel}{d}
-
\frac{v\beta\,(\prel\!\cdot\!\erobp)}{d}.
\label{eq:ddot-split}
\end{align}
\end{lemma}

\begin{proof}
Since $\lambda^{*}$ and $\mu^{*}$ depend on the state only through the two scalars $d$ and $\nv$, the chain rule gives
$\dot\lambda^{*}
=\partial_{d}\lambda^{*}\,\dot d
+\partial_{\nv}\lambda^{*}\,\dot{\nv}$
and similarly for $\mu^{*}$. Collecting the $\dot d$ and $\dot{\nv}$ coefficients of $\dot\lambda^{*}\vry^{2}+\dot\mu^{*}$ yields
\eqref{eq:param-AB}, with
$A(\boldsymbol{x})=\vry^{2}\partial_{d}\lambda^{*}+\partial_{d}\mu^{*}$ and
$B(\boldsymbol{x})=\vry^{2}\partial_{\nv}\lambda^{*}+\partial_{\nv}\mu^{*}$.

Since
$h^{*}=\vrx+\lambda^{*}\vry^{2}+\mu^{*}$ and $\vrx$ is independent of
$d$, $A(\boldsymbol{x}) = \vry^{2}\partial_{d}\lambda^{*}+\partial_{d}\mu^{*}$
is exactly $\partial h^{*}/\partial d$. Using
$\partial_{d}\lambda^{*}=k_{\lambda}/\nv$ and
$\partial_{d}\mu^{*}=k_{\mu}-\kappa/(\gamma\nv)$ gives
\eqref{eq:Adef}.

To compute $B$, write
$\lambda^{*}
=
[k_{\lambda}d-\kappa/(\gamma a_{\max})]\nv^{-1}$.
Since the bracket is independent of $\nv$,
$\partial_{\nv}\lambda^{*}=-\lambda^{*}/\nv$.
Similarly,
$\mu^{*}
=
k_{\mu}d-(\kappa d/\gamma)\nv^{-1}$
gives
$\partial_{\nv}\mu^{*}
=
\kappa d/(\gamma\nv^{2})$.
Therefore
\begin{align*}
B
&=
-\frac{\lambda^{*}\vry^{2}}{\nv}
+\frac{\kappa d}{\gamma\nv^{2}}
\\
&=
-\nv
\left(
\frac{\lambda^{*}\vry^{2}}{\nv^{2}}
-\frac{\kappa d}{\gamma\nv^{3}}
\right)
=
-\sigma\nv,
\end{align*}
which is precisely \eqref{eq:Bdef}, with $\sigma$ defined by
\eqref{eq:sigma}.

We now identify the acceleration-dependent terms. Since $\nv$ is the magnitude of the relative velocity, differentiating the scalar norm gives
$\dot{\nv}
=
\boldsymbol v_{\mathrm{rel}}^{\!\top}
\dot{\boldsymbol{v}}_{\mathrm{rel}}/\nv
=
\boldsymbol v_{\mathrm{rel}}^{\!\top}
(\dot{\boldsymbol v}_{\mathrm{obs}}
-\dot{\boldsymbol v}_{\mathrm{rob|ctrl}})/\nv$.
Substituting $B=-\sigma\nv$ gives
\eqref{eq:sigma-split}. Thus, all acceleration dependence of the
parameter channel enters through the scalar $\sigma(\boldsymbol{x})$.

The clearance rate $\dot d$, by contrast, depends only on positions and velocities. From
$\dot{\boldsymbol p}_{\mathrm{rel}}
=
\boldsymbol v_{\mathrm{rel}}
-
v\beta\,\boldsymbol e_{\mathrm{rob}}^{\perp}$
and
$\dot d
=
\boldsymbol p_{\mathrm{rel}}^{\!\top}
\dot{\boldsymbol p}_{\mathrm{rel}}/d$,
we obtain
\eqref{eq:ddot-split}. The first term is $\dot d_{0}$~\eqref{eq:ddot}, while the second is a steering-slip correction proportional to $\beta$.
\end{proof}

\begin{remark}\label{rem:gap}
Setting $\dot{\boldsymbol{v}}_{\mathrm{obs}}=\boldsymbol{0}$ eliminates the term $-\sigma(\vrel\!\cdot\!\dot{\boldsymbol{v}}_{\mathrm{obs}})$ from \eqref{eq:sigma-split}, recovering the constant-velocity derivation of~\cite{park2026dpcbf}. Since this term is first-order in the obstacle maneuver authority $\kappa$ (cf.~\eqref{eq:sigma}), it represents the first-order coupling
between obstacle acceleration and barrier dynamics.
\end{remark}

\subsection{Grouping the derivative by source}
\label{ssec:group}
Equation~\eqref{eq:hdot-raw} is organized by mechanism, whereas the worst-case analysis requires grouping terms by their source. Substituting Lemmas~\ref{lem:los} and~\ref{lem:sigma}, and decomposing $\dot\alpha$ and $\dot d$ into no-slip and slip components, each term can be assigned to the obstacle inputs, robot inputs, neither, or the slip velocity $v\beta\erobp$. Collecting these terms yields the exact identity
\begin{equation}
\dot h^{*}=\underbrace{\Dstar(\boldsymbol{x})}_{\text{drift}}
+\underbrace{\Lgh\,\boldsymbol{u}}_{\text{control}}
+\underbrace{[\dot h^{*}]^{\mathrm{full}}_{\mathrm{obs}}}_{\text{obstacle}}
+\underbrace{\rho(\boldsymbol{x},\boldsymbol{u})}_{\text{slip residual}},
\label{eq:hdot-grouped}
\end{equation}
whose four terms are as follows.

\emph{Obstacle term.} The obstacle inputs enter through the LoS dynamics~\eqref{eq:losx}--\eqref{eq:losy} and through the term $-\sigma\,\vrel\!\cdot\!\dot{\boldsymbol v}_{\mathrm{obs}}$ of~\eqref{eq:sigma-split}. Expressing $\vrel\!\cdot\!\eobs$ and $\vrel\!\cdot\!\eobsp$ in the LoS frame gives
\begin{align}
[\dot h^{*}]^{\mathrm{full}}_{\mathrm{obs}}
&=a_{\mathrm{obs}}P(\tho)+v_{\mathrm{obs}}\omega_{\mathrm{obs}}Q(\tho),
\label{eq:obsterm}\\
P(\phi)&=(1-\sigma\vrx)\cos\phi+(2\lambda^{*}-\sigma)\vry\sin\phi,
\label{eq:P}\\
Q(\phi)&=(2\lambda^{*}-\sigma)\vry\cos\phi-(1-\sigma\vrx)\sin\phi,
\label{eq:Q}
\end{align}
where $\tho$ is the obstacle heading in the LoS frame.

\emph{Control term.} The robot inputs enter through the LoS dynamics and through the term $+\sigma\,\vrel\!\cdot\!\dot{\boldsymbol v}_{\mathrm{rob|ctrl}}$ of~\eqref{eq:sigma-split}, giving the \emph{principal control term} $\Lgh\boldsymbol{u}=C_{a}'a+C_{\beta}'\beta$, with $\boldsymbol{C}'(\boldsymbol{x}):=(C_{a}',C_{\beta}')^{\top}$ and
\begin{align}
C_{a}'&=\underbrace{-\cos\thr-2\lambda^{*}\vry\sin\thr}_{C_{a}}
+\sigma(\vrel\!\cdot\!\erob),
\label{eq:Cap}\\
\frac{\ell_{r}}{v^{2}}C_{\beta}'&=
\underbrace{\sin\thr-2\lambda^{*}\vry\cos\thr}_{(\ell_{r}/v^{2})C_{\beta}}
+\sigma(\vrel\!\cdot\!\erobp),
\label{eq:Cbp}
\end{align}
where the factor $v^{2}/\ell_{r}$ comes from $\dot\theta=(v/\ell_{r})\beta$. In the multi-obstacle setting of Section~\ref{sec:controller}, $\boldsymbol{C}_{j}'(\boldsymbol{x})$ denotes \eqref{eq:Cap}-\eqref{eq:Cbp} evaluated at the relative state of obstacle $j$.

\emph{Drift term.} With both inputs set to zero, $\dot{\nv}=0$ and only the no-slip clearance and LoS rates survive. The frame-rotation terms of Lemma~\ref{lem:los} combine as $R_{x}+2\lambda^{*}\vry R_{y}=\dot\alpha\,\vry(1-2\lambda^{*}\vrx)$, so
\begin{equation}
\Dstar(\boldsymbol{x})
=
A(\boldsymbol{x})\,\dot d_{0}
+
\dot\alpha_{0}\,\vry\!\left(1-2\lambda^{*}\vrx\right),
\label{eq:Dstar}
\end{equation}
where $\dot d_{0}=\npr\vrx/d$ is the no-slip clearance rate~\eqref{eq:ddot} and $\dot\alpha_{0}=(\prel\times\vrel)/\npr^{2}$ the no-slip LoS angular rate.

\emph{Slip residual.} The lateral slip velocity $v\beta\erobp$ perturbs the clearance and LoS rates by $\dot d_{\mathrm{slip}}=-v\beta(\prel^{\!\top}\erobp)/d$ and $\dot\alpha_{\mathrm{slip}}=-v\beta(\prel\times\erobp)/\npr^{2}$. These enter $\dot h^{*}$ through the same coefficients as the drift:
\begin{equation}
\rho(\boldsymbol{x},\boldsymbol{u})=A(\boldsymbol{x})\,\dot d_{\mathrm{slip}}
+
\dot\alpha_{\mathrm{slip}}\,\vry\,(1-2\lambda^{*}\vrx).
\label{eq:rho}
\end{equation}
Both $\dot d_{\mathrm{slip}}$ and $\dot\alpha_{\mathrm{slip}}$ are proportional to $\beta$, so $\dot h^{*}$ is affine in $\boldsymbol{u}$ with $\beta$-coefficient $C_{\beta}'+\partial_{\beta}\rho$. The principal term $\boldsymbol{C}'(\boldsymbol{x})^{\!\top}\boldsymbol{u}$ is therefore \emph{not} the full Lie derivative $L_{\boldsymbol{g}}h^{*}\boldsymbol{u}$; the two differ by $\partial_{\beta}\rho\,\beta$. Since $\rho$ contributes only a small geometric correction, we do not exploit it as control authority: it is bounded in Section~\ref{sec:residual} and treated as a disturbance, which is conservative because it discards authority the robot in fact possesses.

\subsection{Worst-Case Obstacle Contribution}
\label{ssec:obs}

The obstacle term~\eqref{eq:obsterm} is a combination of the two functions $P$ and $Q$ of the obstacle's LoS heading, weighted by the two obstacle inputs. Both $P$ and $Q$ are sinusoids with the \emph{same} amplitude
$\sqrt{(1-\sigma\vrx)^{2}+(2\lambda^{*}-\sigma)^{2}\vry^{2}}$. Because $\F$ is a
rectangle, the worst $(a_{\mathrm{obs}},\omega_{\mathrm{obs}})$ separates, and
$\sup_{\phi}|A\cos\phi+B\sin\phi|=\sqrt{A^{2}+B^{2}}$ gives the clean bound
\begin{align}
&\sup_{\F,\tho}\bigl(-[\dot h^{*}]^{\mathrm{full}}_{\mathrm{obs}}\bigr)
\le\kappa\,\Lambda(\boldsymbol{x}),\nonumber \; \text{where} \\ 
&\Lambda(\boldsymbol{x}):=\sqrt{(1-\sigma\vrx)^{2}+(2\lambda^{*}-\sigma)^{2}\vry^{2}}.
\label{eq:Lambda}
\end{align}
Thus, $\Lambda(\boldsymbol{x})$ is the geometric gain through which the obstacle acts on $\dot{h} ^ {*}$. Whatever the obstacle chooses to do, its effect on $\dot h^{*}$ is at most its scalar capability $\kappa$ scaled by $\Lambda(\boldsymbol{x})$, a quantity fixed by the current geometry alone. $\Lambda$ is large where the barrier is sensitive to the relative velocity and small where it is nearly flat. A bound of the form $\kappa\Lambda$ would be of little use if $\Lambda$ could vanish, because the robot's own authority is scaled by the same factor (Lemma~\ref{lem:id}). At such a state, neither agent could affect the barrier, rendering the comparison vacuous. Lemma~\ref{lem:Lmin} in Appendix~\ref{app:proofs} rules out this degeneracy by establishing $\Lambda_{\min}>0$ on $\mathcal{B}$.


\subsection{Control Authority}
\label{ssec:auth}

We now bound the control term $\Lgh\boldsymbol{u}=C_{a}'a+C_{\beta}'\beta$, with coefficients~\eqref{eq:Cap}-\eqref{eq:Cbp}. The decisive fact is that the robot's gain has the \emph{same amplitude} $\Lambda$ as the obstacle's.

\begin{lemma} \label{lem:id}
Robot and obstacle share $\Lambda$\ i.e., $C_{a}'^{2} + \bigl(\tfrac{\ell_{r}}{v^{2}}C_{\beta}'\bigr)^{2} = \Lambda(\boldsymbol{x})^{2}$.
\end{lemma}
\begin{proof}
Let
$\bar C_{a}:=C_{a}$ and
$\bar C_{\beta}:=(\ell_{r}/v^{2})C_{\beta}$, and define
$s:=\boldsymbol v_{\mathrm{rel}}^{\!\top}\boldsymbol e_{\mathrm{rob}}$ and $q:=\boldsymbol v_{\mathrm{rel}}^{\!\top}\boldsymbol e_{\mathrm{rob}}^{\perp}$.
Then $s^{2}+q^{2}=\nv^{2}$ and, from
\eqref{eq:Cap}-\eqref{eq:Cbp},
$C_{a}'=\bar C_{a}+\sigma s$ and
$(\ell_{r}/v^{2})C_{\beta}'=\bar C_{\beta}+\sigma q$.
Expanding the squared norm gives
\[
C_{a}'^{2}
+
\Bigl(\frac{\ell_{r}}{v^{2}}C_{\beta}'\Bigr)^{2}
=
(\bar C_{a}^{2}+\bar C_{\beta}^{2})
+
2\sigma(\bar C_{a}s+\bar C_{\beta}q)
+
\sigma^{2}\nv^{2}.
\]
The first group simplifies directly from the definitions of $\bar C_{a}$ and $\bar C_{\beta}$: $\bar C_{a}^{2}+\bar C_{\beta}^{2} = 1+4(\lambda^{*})^{2}\vry^{2}.$ For the cross-term, writing $\boldsymbol e_{\mathrm{rob}}
=
[\cos\tilde\theta,\sin\tilde\theta]^{\top}$ and
$\boldsymbol e_{\mathrm{rob}}^{\perp}
=
[-\sin\tilde\theta,\cos\tilde\theta]^{\top}$,
the mixed $\sin\tilde\theta\cos\tilde\theta$ terms cancel, yielding
\[
\bar C_{a}s+\bar C_{\beta}q
=
-(\vrx+2\lambda^{*}\vry^{2}).
\]
Substituting these identities and using
$\nv^{2}=\vrx^{2}+\vry^{2}$ gives
\begin{align}
C_{a}'^{2}
+
\Bigl(\frac{\ell_{r}}{v^{2}}C_{\beta}'\Bigr)^{2}
=
1+4(\lambda^{*})^{2}\vry^{2}
-2\sigma(\vrx+2\lambda^{*}\vry^{2}) \nonumber\\
+\sigma^{2}(\vrx^{2}+\vry^{2}). \nonumber
\end{align}
Regrouping terms yields, $(1-\sigma\vrx)^{2}
+ (2\lambda^{*}-\sigma)^{2}\vry^{2}
=\Lambda^{2},$
where the last equality follows from \eqref{eq:Lambda}.
\end{proof}

The common gain $\Lambda$ therefore cancels from the adversarial comparison, reducing the problem from optimization over obstacle strategies to the scalar comparison $\cmin$ versus $\kappa$. Note that this conclusion holds for any certificate that is affine in $\vrx$, quadratic in $\vry$, and whose parameters depend on the state only through $d$ and $\nv$, with $\sigma$ defined by the corresponding partial derivatives in~\eqref{eq:sigma}.

\begin{proposition}\label{prop:auth}
Let $\cmin:=\min\{\amax,\,v_{\min}^{2}\beta_{\max}/\ell_{r}\}$. Then the robot's
maximal effect on $\dot h^{*}$ is
\begin{equation}
\Phi(\boldsymbol{x}){:=}\sup_{\boldsymbol{u}\in\U}\Lgh \boldsymbol{u}{=}|C_{a}'|\amax{+} |C_{\beta}'|\beta_{\max}
\ {\ge}\ \cmin\,\Lambda(\boldsymbol{x}).
\label{eq:auth}
\end{equation}
\end{proposition}

\begin{proof}
Since $L_g h^*\,\boldsymbol{u}=C_a'a+C_\beta'\beta$ is linear in the bounded inputs, its maximum is attained by choosing $a$ and $\beta$ at their respective bounds with signs matching those of $C_a'$ and $C_\beta'$. Thus, $\sup_{\boldsymbol{u}\in\mathcal{U}} L_g h^*\,\boldsymbol{u} = |C_a'|a_{\max}+|C_\beta'|\beta_{\max}.$
Factor the steering term as
$|C_{\beta}'|\beta_{\max}=\tfrac{v^{2}}{\ell_{r}}\beta_{\max}\cdot
\tfrac{\ell_{r}}{v^{2}}|C_{\beta}'|$. Using the elementary inequality
$c_{1}A+c_{2}B\ge\min(c_{1},c_{2})(A+B)$ for $A,B\ge0$ with $c_{1}=\amax$, $c_{2}=\tfrac{v^{2}}{\ell_{r}}\beta_{\max}$, and noting $\tfrac{v^{2}}{\ell_{r}}\beta_{\max}\ge v_{\min}^{2}\beta_{\max}/\ell_{r}$,
\[
\Phi\ge\cmin\Bigl(|C_{a}'|+\tfrac{\ell_{r}}{v^{2}}|C_{\beta}'|\Bigr)
\ge\cmin\sqrt{C_{a}'^{2}+\bigl(\tfrac{\ell_{r}}{v^{2}}C_{\beta}'\bigr)^{2}}
=\cmin\Lambda,
\]
where we use $|X|+|Y|\ge\sqrt{X^{2}+Y^{2}}$ and Lemma~\ref{lem:id}. At $v=0$, $\cmin=0$, hence Assumption~\ref{as:rob} requires $v_{\min}>0$.
\end{proof}

The obstacle and control contributions admit complementary bounds with the common factor $\Lambda$. The remaining terms in $\dot{h}^*$ \eqref{eq:hdot-grouped} are independent of both obstacle maneuvers and control inputs.

\subsection{Establishing a Positive Drift Margin}
\label{ssec:drift}

The drift~\eqref{eq:Dstar} is what happens to $h^{*}$ when neither agent acts i.e., the barrier still moves, because the clearance and the LoS angle keep evolving. The next lemma records that this motion cannot be arbitrarily adverse on $\mathcal{B}$, and gives an explicit bound in terms of the problem constants.

\begin{lemma}\label{lem:drift}
Under Assumptions~\ref{as:obs}-\ref{as:clear} and the bounds of Proposition~\ref{prop:bounds}, the autonomous drift $\Dstar$ attains a finite infimum on $\mathcal{B}=\partial\Cset\cap\Omega\cap\{d\ge r\}$,
\begin{equation}
\Dmin:=\inf_{x\in\mathcal{B}}\Dstar(\boldsymbol{x})>-\infty,
\label{eq:Dmin}
\end{equation}
and satisfies the explicit lower bound
$\Dmin\ge-\bar{\mathcal D}$, where
\begin{equation}
\bar{\mathcal D}=M_{d}\,\frac{p_{\max}v_{\mathrm{rel,max}}}{d_{\min}}
+M_{\alpha}\,\frac{v_{\mathrm{rel,max}}}{p_{\min}},
\label{eq:Dbar}
\end{equation}
with
$M_{d}=k_{\lambda}v_{\mathrm{rel,max}}+k_{\mu}
+\kappa/(\gamma v_{\mathrm{rel,min}})$
bounding $|A|$,
$M_{\alpha}=v_{\mathrm{rel,max}}
(1+2\lambda^{*}_{\max}v_{\mathrm{rel,max}})$
bounding $|\vry(1-2\lambda^{*}\vrx)|$,
$\lambda^{*}_{\max}=k_{\lambda}d_{\max}/v_{\mathrm{rel,min}}$,
$d_{\min}=r$,
$p_{\min}=\sqrt{2}\,r$,
$p_{\max}=\sqrt{d_{\max}^{2}+r^{2}}$, and
$v_{\mathrm{rel,max}}=v_{\max}+\vobm$.
\end{lemma}

\begin{proof}
\emph{Existence.} By Proposition~\ref{prop:bounds} the set $\mathcal{B}$ is
compact and on it $d\ge r$, $\npr\ge\sqrt{2}\,r$ and
$\nv\ge v_{\mathrm{rel,min}}$ bound every denominator in
\eqref{eq:Dstar} away from zero, so $\Dstar$ is continuous on $\mathcal{B}$. A continuous function on a non-empty compact set attains its minimum, so the infimum $\Dmin$ is finite.

\emph{Explicit bound.} We bound each term of \eqref{eq:Dstar}. From
\eqref{eq:ddot}, $|\dot d_{0}|=\npr|\vrx|/d\le p_{\max}v_{\mathrm{rel,max}}/d_{\min}$,
and $|A|\le k_{\lambda}v_{\mathrm{rel,max}}+|k_{\mu}-\kappa/(\gamma\nv)|\le M_{d}$,
so the clearance drift is at most $M_{d}\,p_{\max}v_{\mathrm{rel,max}}/d_{\min}$.
For the rotation, $|\dot\alpha_{0}|=|\prel\times\vrel|/\npr^{2}\le\nv/\npr\le
v_{\mathrm{rel,max}}/p_{\min}$, and $|\vry(1-2\lambda^{*}\vrx)|\le M_{\alpha}$,
so the rotation drift is at most $M_{\alpha}v_{\mathrm{rel,max}}/p_{\min}$. Adding
the two gives $|\Dstar|\le\bar{\mathcal D}$, hence $\Dmin\ge-\bar{\mathcal D}$.
\end{proof}

\subsection{Bounding the Slip Residual}
\label{sec:residual}

The slip residual~\eqref{eq:rho} is the price of the small-slip model. It is a fixed overhead proportional to $v_{\max}\beta_{\max}$ and vanishes as the slip limit does.

To bound~\eqref{eq:rho}, we use
$|\boldsymbol p_{\mathrm{rel}}^{\!\top}
\boldsymbol e_{\mathrm{rob}}^{\perp}|
\le \|\boldsymbol p_{\mathrm{rel}}\|$,
$|\boldsymbol p_{\mathrm{rel}}
\times
\boldsymbol e_{\mathrm{rob}}^{\perp}|
\le \|\boldsymbol p_{\mathrm{rel}}\|$,
$v\le v_{\max}$,
$|\beta|\le\beta_{\max}$,
$d\ge d_{\min}$,
$\|\boldsymbol p_{\mathrm{rel}}\|\le p_{\max}$,
$\|\boldsymbol p_{\mathrm{rel}}\|\ge p_{\min}$,
$|A(\boldsymbol{x})|\le M_d$, and
$|\vry(1-2\lambda^{*}\vrx)|\le M_\alpha$,
we obtain
$|\dot d_{\mathrm{slip}}|
\le
v_{\max}\beta_{\max}p_{\max}/d_{\min}$
and
$|\dot\alpha_{\mathrm{slip}}|
\le
v_{\max}\beta_{\max}/p_{\min}$.
Substitution into~\eqref{eq:rho} gives
\begin{equation}
|\rho(\boldsymbol{x},\boldsymbol{u})|
\le
\varepsilon
:=
v_{\max}\beta_{\max}
\left(
M_d\frac{p_{\max}}{d_{\min}}
+
\frac{M_\alpha}{p_{\min}}
\right).
\label{eq:eps}
\end{equation}

Combining the obstacle bound~\eqref{eq:Lambda}, the drift bound~\eqref{eq:Dmin},
and the residual bound~\eqref{eq:eps} with
\eqref{eq:hdot-grouped} yields
\begin{equation}
\dot h^{*}
\ge
\Dmin
+
\Lgh \boldsymbol{u}
-
\kappa\Lambda(\boldsymbol{x})
-
\varepsilon,
\label{eq:assembled}
\end{equation}
for every $\boldsymbol{x}\in\mathcal{B}$, every obstacle in $\F$, and every
admissible control $\boldsymbol{u}\in\U$.


\subsection{Sufficient Condition for Safety}
\label{ssec:main}

The preceding analysis establishes a sufficient condition for the feasibility of the runtime safety program at every boundary state within the certified engagement envelope, thereby yielding a conditional forward-invariance guarantee.


\begin{theorem}
\label{thm:valid}
Suppose Assumptions~\ref{as:obs}-\ref{as:clear} hold, let the design
parameters satisfy~\eqref{eq:floor} and~\eqref{eq:Lmin-cond}, and let
$\Dmin$ (Lemma~\ref{lem:drift}), $\varepsilon$
\eqref{eq:eps}, and $\Lmin>0$ (Lemma~\ref{lem:Lmin}) be the
corresponding constants over $\mathcal{B} = \partial\Cset\cap\Omega\cap\{d\ge r\}$. Assume $\psi\in\mathcal{K}_{\infty}$ is locally Lipschitz. If the robot out-authorizes the obstacle,
$\kappa<\cmin$, and
\begin{equation}
\boxed{\ (\cmin{-}\kappa)\,\Lmin\ {\ge}\ \varepsilon{-}\Dmin\ }\,,
\text{i.e.} \; \cmin{\ge} \kappa {+} \frac{\varepsilon{-}\Dmin}{\Lmin},
\label{eq:feas}
\end{equation}
then the following hold.
\begin{enumerate}
\item[(i)] \emph{A-CBF property.} $h^{*}$ satisfies the A-CBF condition of Definition~\ref{def:acbf} at every $\boldsymbol{x}\in\mathcal{B}$.
\item[(ii)] \emph{Feasibility on $\mathcal{B}$.} For a single obstacle, the constraint set of the runtime program~\eqref{eq:base_qp} is non-empty at every $\boldsymbol{x}\in\mathcal{B}$.
\item[(iii)] \emph{Conditional invariance.} Suppose in addition that the closed-loop solution is well defined and that the runtime program remains feasible along it, so that $\dot h^{*}\ge-\psi(h^{*})$ is enforced on $[0,t_{1}]$. Then for every initial condition $\boldsymbol{x}(0)\in\operatorname{int}\Cset$ with $d(0)>r$ and every obstacle strategy in $\F$, and for all $t_{1}\ge0$ such that $\boldsymbol{x}(\tau)\in\Omega$ for all $\tau\in[0,t_{1}]$,
\begin{equation}
h^{*}(\boldsymbol{x}(t_{1}))\ge0
\qquad\text{and}\qquad
d(t_{1})\ge r.
\label{eq:cond_inv}
\end{equation}
Equivalently, the closed-loop trajectory cannot leave $\Cset$ through $\partial\Cset$ while in $\Omega$; any exit from $\Cset\cap\Omega$ occurs through $\partial\Omega$.
\end{enumerate}
\end{theorem}

\begin{proof}
\emph{(i)} For any $\boldsymbol{x}\in\mathcal B=\partial\Cset\cap\Omega\cap\{d\ge r\}$, \eqref{eq:assembled}, together with the bounds of Lemma~\ref{lem:drift}, Proposition~\ref{prop:auth}, and \eqref{eq:Lambda}, gives
\begin{align}
\sup_{\boldsymbol{u}\in\U}\inf_{\F}\dot h^{*}
&\ge\Dmin+\cmin\Lambda(\boldsymbol{x})-\kappa\Lambda(\boldsymbol{x})-\varepsilon\nonumber\\
&=(\cmin-\kappa)\,\Lambda(\boldsymbol{x})-(\varepsilon-\Dmin)\nonumber\\
&\ge(\cmin-\kappa)\,\Lmin-(\varepsilon-\Dmin)\ \ge\ 0.
\label{eq:nagumo}
\end{align}
The third line uses $\Lambda(\boldsymbol{x})\ge\Lmin$ and $\cmin>\kappa$, while the last inequality is precisely \eqref{eq:feas}. This establishes the A-CBF condition of Definition~\ref{def:acbf}.

\emph{(ii)} At $\boldsymbol{x}\in\mathcal{B}$ we have $h^{*}=0$ and hence $\psi(h^{*})=0$, so the constraint in~\eqref{eq:base_qp} reads $\boldsymbol{C}'(\boldsymbol{x})^{\!\top}\boldsymbol{u}\ge\kappa\Lambda(\boldsymbol{x})+\varepsilon-\Dstar(\boldsymbol{x})$. By Proposition~\ref{prop:auth}, $\sup_{\boldsymbol{u}\in\U}\boldsymbol{C}'(\boldsymbol{x})^{\!\top}\boldsymbol{u}\ge\cmin\Lambda(\boldsymbol{x})$, while $\Dstar(\boldsymbol{x})\ge\Dmin$ and $\Lambda(\boldsymbol{x})\ge\Lmin$ give
$\cmin\Lambda-\kappa\Lambda\ge(\cmin-\kappa)\Lmin\ge\varepsilon-\Dmin\ge\varepsilon-\Dstar(\boldsymbol{x})$.
Hence the supremizing input satisfies the constraint and the feasible set is non-empty.

\emph{(iii)} By hypothesis the closed loop enforces $\dot h^{*}\ge-\psi(h^{*})$ on $[0,t_{1}]$. Since $\psi\in\mathcal K_{\infty}$ is locally Lipschitz, the comparison lemma yields $h^{*}(\boldsymbol{x}(t))\ge z(t)$, where $\dot z=-\psi(z)$ and $z(0)=h^{*}(\boldsymbol{x}(0))>0$; the solution $z$ remains strictly positive for every finite time, so $h^{*}(\boldsymbol{x}(t))>0$ on $[0,t_{1}]$. Proposition~\ref{prop:safe} together with Assumption~\ref{as:clear} then yields $d(t)\ge r$ on $[0,t_{1}]$, which is \eqref{eq:cond_inv}. Consequently the trajectory cannot cross $\partial\Cset$ while it remains in $\Omega$, so any exit from $\Cset\cap\Omega$ must occur through $\partial\Omega$.
\end{proof}

\begin{remark}
\label{rem:feas_scope}
Part~(ii) is established on $\mathcal{B}$ only, because $\Lmin$ and $\Dmin$ are boundary constants: $\Lmin$ relies on the boundary condition $h^{*}=0$ (Lemma~\ref{lem:Lmin}). Feasibility in the interior of $\Cset$, and for $N_{\mathrm{obs}}>1$, where the intersection of the per-obstacle constraint sets may be empty, is therefore not implied by \eqref{eq:feas} and enters (iii) as a hypothesis. This is precisely the gap addressed by the soft and buffered formulations of Section~\ref{sec:controller}.
\end{remark}

\begin{corollary}
\label{cor:ics}
Let the hypotheses of Theorem~\ref{thm:valid} hold and let
\[
\Omega_{\infty}:=\bigl\{\boldsymbol{x}(0)\in\Cset\cap\Omega \;:\; \boldsymbol{x}(t)\in\Omega \ \ \forall t\ge0 \bigr\}
\]
denote the initial states whose closed-loop trajectory remains in the engagement envelope. If in addition $d(0)>r$, then
\begin{equation}
\Omega_{\infty}\ \subseteq\ \mathcal{X}\setminus\mathcal{I},
\label{eq:ics_inclusion}
\end{equation}
that is, no state of $\Omega_{\infty}$ is an inevitable collision state.
\end{corollary}

\begin{proof}
Definition~\ref{def:ICS} declares $\boldsymbol{x}(0)$ an ICS only if \emph{every} robot control admits an $\F$-admissible disturbance trajectory producing $\|\boldsymbol{p}(t)-\boldsymbol{p}_{\mathrm{obs}}(t)\|<r$ for some $t>0$. It therefore suffices to exhibit one admissible control that avoids this for all disturbances. Take the closed-loop control of Theorem~\ref{thm:valid}, which is admissible by construction since $\boldsymbol{u}\in\U$. For $\boldsymbol{x}(0)\in\Omega_{\infty}$ the trajectory remains in $\Omega$ for all $t\ge0$, so~\eqref{eq:cond_inv} applies with arbitrary $t_{1}$ and yields $d(t)\ge r$, i.e. $\|\boldsymbol{p}(t)-\boldsymbol{p}_{\mathrm{obs}}(t)\|\ge r$, for every $t\ge0$ and every obstacle strategy in $\F$. Hence $\boldsymbol{x}(0)\notin\mathcal{I}$.
\end{proof}


\label{sec:estimator}

Theorem~\ref{thm:valid} requires the capability bounds
$a_{\mathrm{obs,max}}$, $\omega_{\mathrm{obs,max}}$, and $v_{\mathrm{obs,max}}$ to be known. In practice, they must be estimated online from noisy observations. This section provides a sliding-window estimator and shows that, with high probability, the estimated bounds dominate the true ones, so the conservative barrier built from them inherits the safety guarantee of Theorem~\ref{thm:valid}.

\subsection{Sliding-Window Estimator}
\label{subsec:estimator}

From a sliding window of $N_{w}$ observations at times
$t-N_{w}\Delta t,\ldots,t$, define the running maxima
\begin{align}
  \hat{a}_{\mathrm{obs,max}}(t)
  &= \max_{k\in\mathrm{window}}
     \frac{\lvert v_{\mathrm{obs}}(t_{k})-v_{\mathrm{obs}}(t_{k-1})\rvert}
          {\Delta t},
  \label{eq:ahat}\\
  \hat{\omega}_{\mathrm{obs,max}}(t)
  &= \max_{k\in\mathrm{window}}
     \frac{\bigl\lvert \operatorname{wrap}\bigl(\theta_{\mathrm{obs}}(t_{k})-\theta_{\mathrm{obs}}(t_{k-1})\bigr)\bigr\rvert}
          {\Delta t},
  \label{eq:omegahat}\\
  \hat{v}_{\mathrm{obs,max}}(t)
  &= \max_{k\in\mathrm{window}}
     \lvert v_{\mathrm{obs}}(t_{k})\rvert ,
  \label{eq:vmaxhat}
\end{align}
Here $\operatorname{wrap}(\cdot)$ maps an angle difference to $(-\pi,\pi]$, which is required because headings are measured modulo $2\pi$. Each finite difference in \eqref{eq:ahat}-\eqref{eq:omegahat} estimates the corresponding rate averaged over one sampling interval. Consequently, in the noiseless case, the running maxima equal the largest interval-averaged acceleration, turning rate, and sampled speed observed over the window, and therefore never underestimate these realized interval-averaged quantities. Intra-sample peaks of shorter duration, however, may exceed the corresponding averages. Assumption~\ref{asm:realization} relates these averaged maxima to the capability bounds, while the margins $\varepsilon_a$, $\varepsilon_\omega$, and $\varepsilon_v$ compensate for sensor uncertainty.
\begin{equation}
\begin{aligned}
\tilde{a}_{\mathrm{obs,max}}
&=\hat{a}_{\mathrm{obs,max}}+\varepsilon_a,\;
\tilde{\omega}_{\mathrm{obs,max}}
=\hat{\omega}_{\mathrm{obs,max}}+\varepsilon_\omega,\\
\tilde{v}_{\mathrm{obs,max}}
&=\hat{v}_{\mathrm{obs,max}}+\varepsilon_v.
\end{aligned}
\label{eq:margins}
\end{equation}

The estimator can only certify the capability that the obstacle reveals. Linking the realized window maxima to the capability bounds of Definition~\ref{def:F}, therefore, requires an exposure assumption.

\begin{assumption}
\label{asm:realization}
Within each estimation window, the obstacle realizes its capability at the sampling resolution: the interval-averaged longitudinal acceleration and turning rate attain $a_{\mathrm{obs,max}}$ and
$\omega_{\mathrm{obs,max}}$, respectively, over at least one sampling interval, and the sampled speed attains
$v_{\mathrm{obs,max}}$ at least once. Thus, the bounds in
Theorem~\ref{thm:valid} are interpreted as realized
interval-averaged capabilities. Maneuvers shorter than one sampling interval are certified only through their interval averages.
\end{assumption}

\subsection{Robustness Guarantee}
\label{subsec:robust}

\begin{proposition}
\label{prop:robust}
Suppose the measurements $v_{\mathrm{obs}}$ and
$\theta_{\mathrm{obs}}$ are corrupted by independent zero-mean sub-Gaussian noise with proxy parameters
$\sigma_v$ and $\sigma_\theta$, respectively, and that
Assumption~\ref{asm:realization} holds. For any
$\delta\in(0,1)$, define
\begin{align}
\varepsilon_a
&=
\frac{\sqrt{2}\sigma_v}{\Delta t}\,\eta,
&
\varepsilon_\omega
&=
\frac{\sqrt{2}\sigma_\theta}{\Delta t}\,\eta,
&
\varepsilon_v
&=
\sigma_v\,\eta,
\label{eq:eps_prob}
\end{align}
where, $\eta:= \sqrt{2\ln\!\left({3N_w}/{\delta}\right)}.$
Then, at any fixed estimator update, with probability at least $1-\delta$, 
$\tilde a_{\mathrm{obs,max}}
= \hat a_{\mathrm{obs,max}}+\varepsilon_a
\ge a_{\mathrm{obs,max}},
\tilde\omega_{\mathrm{obs,max}}
= \hat\omega_{\mathrm{obs,max}}+\varepsilon_\omega
\ge \omega_{\mathrm{obs,max}},
\tilde v_{\mathrm{obs,max}}
= \hat v_{\mathrm{obs,max}}+\varepsilon_v
\ge v_{\mathrm{obs,max}},$ and therefore, 
$\tilde\kappa = \tilde a_{\mathrm{obs,max}}
+ \tilde v_{\mathrm{obs,max}}
\tilde\omega_{\mathrm{obs,max}}
\ge \kappa.$ Consequently, the AR-DPCBF constructed from $\tilde\kappa$ defines a conservative barrier satisfying
$\tilde h^{*}\le h^{*}\le h,
\;
\tilde{\mathcal C}^{*}
\subseteq
\mathcal C^{*}
\subseteq
\mathcal C.$ Hence, with probability at least $1-\delta$,
(i) every controller maintaining $\tilde h^{*}\ge0$
also maintains $h^{*}\ge0$ and $h\ge0$; and
(ii) if $\tilde\kappa$ satisfies the hypotheses of
Theorem~\ref{thm:valid}, then $\tilde h^{*}$ is a valid
A-CBF for the true obstacle.
\end{proposition}

\begin{proof}
We first establish high-probability coverage of the obstacle capabilities. Each acceleration estimate in \eqref{eq:ahat} is the difference of two independent sub-Gaussian measurements and therefore has sub-Gaussian proxy
$
s_a=\sqrt{2}\sigma_v/\Delta t.
$
Writing the $k$-th differenced sample as
$d_k+n_k$, Assumption~\ref{asm:realization} guarantees an index
$k^\star$ satisfying
$d_{k^\star}=a_{\mathrm{obs,max}}$.
Hence,
\begin{align}
\Pr\!\left(
\hat a_{\mathrm{obs,max}}{<}a_{\mathrm{obs,max}}{-}
\varepsilon_a
\right)
&\le
\Pr\!\left(
\max_k(-n_k)
>
\varepsilon_a
\right)
\nonumber\\
&\le
N_w
\exp\!\left(
{-}\frac{\varepsilon_a^2}{2s_a^2}
\right)
{=}
\frac{\delta}{3},
\label{eq:acc_prob}
\end{align}
where the last equality follows from
\eqref{eq:eps_prob}. Therefore,
$
\Pr( \tilde a_{\mathrm{obs,max}}
{\ge}
a_{\mathrm{obs,max}})
\ge
1-(\delta/3).
$

The same argument applies to $\hat\omega_{\mathrm{obs,max}}$ with $s_\omega=\sqrt2\sigma_\theta/\Delta t$,
while the speed estimate requires no differencing and uses proxy
$\sigma_v$, yielding
$
\Pr(
\tilde\omega_{\mathrm{obs,max}}
\ge
\omega_{\mathrm{obs,max}})
\ge
1-\delta/3
$
and
$
\Pr(
\tilde v_{\mathrm{obs,max}}
\ge
v_{\mathrm{obs,max}})
\ge
1-\delta/3.
$
A union bound gives the required results. On this event,
$\tilde\kappa = \tilde a_{\mathrm{obs,max}}
+ \tilde v_{\mathrm{obs,max}}
\tilde\omega_{\mathrm{obs,max}}
\ge \kappa$ follows immediately. Since the contractions
$\Delta\lambda$
and
$\Delta\mu$
of Theorem~\ref{thm:params}
are monotone increasing in $\kappa$,
$\tilde\kappa\ge\kappa
\Longrightarrow
\Delta\tilde\lambda
\ge
\Delta\lambda,
\;
\Delta\tilde\mu
\ge
\Delta\mu,
\Longrightarrow
\tilde\lambda^{*}
\le
\lambda^{*},
\;
\tilde\mu^{*}
\le
\mu^{*},
\Longrightarrow
\tilde h^{*}
\le
h^{*},
\;
\tilde{\mathcal C}^{*}
\subseteq
\mathcal C^{*},$
which together with Proposition~\ref{prop:safe} proves
Part~(i).

For Part~(ii), the true disturbance set satisfies
$\F\subseteq\tilde\F$.
Hence, if the inflated capability
$\tilde\kappa$
satisfies the hypotheses of
Theorem~\ref{thm:valid},
then that theorem certifies invariance of
$\tilde{\mathcal C}^{*}$
against every disturbance in
$\tilde\F$,
and therefore against the true obstacle.
The result follows by combining this guarantee with the coverage
event, whose probability is at least
$1-\delta$.
\end{proof}


\section{Controller Formulation and Variants}
\label{sec:controller}
This section presents the certified adversarial controller as a single QP, along with two soft variants that trade formal guarantees for feasibility. 

\subsection{Base AR-DPCBF Quadratic Program}
\label{subsec:base}

For each obstacle $j$, let $\kappa_j$ denote a conservative upper bound on its maneuvering capability. Using $\kappa_j$, the contracted DPCBF parameters $(\lambda_j^*,\mu_j^*)$ are computed according to Theorem~\ref{thm:params}, yielding the adversarial barrier, $h_j^*(\boldsymbol{x}) = \tilde v_{{\rm rel},x}^{\,j} + \lambda_j^*
\big(\tilde v_{{\rm rel},y}^{\,j}\big)^2 + \mu_j^*$. The control input is obtained by,
\begin{align}
\boldsymbol{u}^*
&=
\arg\min_{\boldsymbol{u}\in\mathcal U}
\;
\|\boldsymbol{u}-\boldsymbol{u}_{\rm ref}\|^2
\label{eq:base_qp}\\
\text{s.t.}\quad
&
\boldsymbol{C}_j'(\boldsymbol{x})^{\!\top}\boldsymbol{u}
\ge
\kappa_j\Lambda_j(\boldsymbol{x})
+
\varepsilon_\rho
-
D_j^*(\boldsymbol{x})
-
\psi\!\left(h_j^*(\boldsymbol{x})\right),
\nonumber\\
&
\hfill j=1,\ldots,N_{\rm obs},
\nonumber
\end{align}
where $\Lambda_j(\boldsymbol{x})$ and $D_j^*(\boldsymbol{x})$ are given by \eqref{eq:Lambda} and \eqref{eq:Dstar}, respectively, and
$\varepsilon_\rho:=\varepsilon$ is the residual bound
\eqref{eq:eps}. The resulting optimization is a convex quadratic program with two decision variables and $N_{\rm obs}$ linear constraints, referred to as the \emph{Hard AR-DPCBF}.

When the obstacle capability is unknown,
$\kappa_j$ is replaced by the online estimate
$\tilde{\kappa}_j=
\tilde a_{\mathrm{obs,max}}^{\,j} +
\tilde v_{\mathrm{obs,max}}^{\,j}
\tilde\omega_{\mathrm{obs,max}}^{\,j}$,
computed by the estimator of Section~\ref{sec:estimator}. By
Proposition~\ref{prop:robust}, this substitution preserves the safety guarantee with probability at least $1-\delta$.

Algorithm~\ref{alg:ardpcbf} summarizes the resulting control cycle,
which is common to all three AR-DPCBF variants and differs only in the
optimization problem solved at the final step.

\begin{algorithm}[t]
\caption{AR-DPCBF control cycle}
\label{alg:ardpcbf}
\begin{algorithmic}[1]
\REQUIRE Robot state $\boldsymbol{x}$, obstacle observations
$\{(\boldsymbol{p}_{\mathrm{obs}}^j,v_{\mathrm{obs}}^j,
\theta_{\mathrm{obs}}^j)\}_{j=1}^{N_{\mathrm{obs}}}$,
reference $\boldsymbol{u}_{\mathrm{ref}}$, variant
$V\in\{\textsc{Hard},\textsc{Soft},\textsc{Buffer}\}$

\ENSURE Control $\boldsymbol{u}^*=(a^*,\beta^*)$

\FOR{$j=1,\ldots,N_{\mathrm{obs}}$}
    \STATE Update the capability estimate $\tilde{\kappa}_j$
    using the sliding window and confidence margins
    \STATE Compute the LoS relative state and contracted parameters
    $(\lambda_j^*,\mu_j^*)$
    \STATE Construct the contracted barrier $h_j^*$ 
\ENDFOR

\IF{$V=\textsc{Hard}$}
    \STATE $\boldsymbol{u}^*\leftarrow$ solve~\eqref{eq:base_qp}
    \COMMENT{adversarial constraint enforced}
\ELSIF{$V=\textsc{Soft}$}
    \STATE $\boldsymbol{u}^*\leftarrow$ solve~\eqref{eq:soft_qp}
    \COMMENT{adversarial residual penalized}
\ELSE
    \STATE $\boldsymbol{u}^*\leftarrow$ solve~\eqref{eq:buffer_qp}
    \COMMENT{buffered penalty, width $\varepsilon_b$}
\ENDIF

\RETURN $\boldsymbol{u}^*$
\end{algorithmic}
\end{algorithm}

\begin{remark}
\label{rem:complexity}
Lines~3-4 of Algorithm~\ref{alg:ardpcbf} are $O(1)$ arithmetic per obstacle.
The sliding-window update (line~2) is $O(N_w)$ per obstacle, giving
$O(N_{\mathrm{obs}}N_w)$ for the estimator. The QP (lines~6-12) has
$N_{\mathrm{obs}}$ linear constraints in two decision variables and is solvable
in $O(N_{\mathrm{obs}})$ with warm-starting. The total per-cycle cost is
therefore $O(N_{\mathrm{obs}}N_w)$, identical in order to the DPCBF controller.
\end{remark}

\subsection{Conservatism of Hard AR-DPCBF}
\label{subsec:hard_conservatism}

Since $\lambda_j^*\le\lambda_j$ and $\mu_j^*\le\mu_j$, with the contraction increasing monotonically with $\kappa_j$ (Theorem~\ref{thm:params}), the adversarial safe set
$\{h_j^*\ge0\}$ is a strict subset of the nominal DPCBF safe set whenever $\kappa_j>0$. Consequently, the robust constraint in \eqref{eq:base_qp} is more restrictive than its DPCBF counterpart. Let $\mathcal{K}_{\mathrm{cbf},j}^*(\boldsymbol{x})$ denote the corresponding admissible-input set for obstacle $j$. In dense multi-obstacle environments, the intersection $\bigcap_j\mathcal{K}_{\mathrm{cbf},j}^*(\boldsymbol{x})$ may become empty, rendering \eqref{eq:base_qp} infeasible. In contrast, the nominal DPCBF remains feasible by construction because it neglects adversarial obstacle maneuvers. This conservatism motivates the soft AR-DPCBF formulations presented next, which preserve adversarial awareness while recovering feasibility whenever the hard formulation becomes infeasible.

\subsection{Soft AR-DPCBF}
\label{subsec:soft}

Since hard AR-DPCBF is likely to become infeasible, we therefore retain the nominal DPCBF constraint as a hard safety constraint and relax only the adversarial condition by penalizing its residual in the objective. For each obstacle, define the robust residual
\begin{equation}
\label{eq:robust_residual}
g_j(\boldsymbol{x},\boldsymbol{u})
{:=}
\boldsymbol{C}_j'(\boldsymbol{x})^{\!\top}\boldsymbol{u}
{+}
D_j^*(\boldsymbol{x})
{-}
\kappa_j\Lambda_j(\boldsymbol{x})
{-}
\varepsilon_\rho
{+}
\psi\!\left(h_j^*(\boldsymbol{x})\right),
\end{equation}
which is affine in $\boldsymbol{u}$ and satisfies
$g_j(\boldsymbol{x},\boldsymbol{u})\ge0$
if and only if the adversarial CBF constraint in
\eqref{eq:base_qp} holds.

\begin{definition}[Soft AR-DPCBF]
\label{def:soft}
The Soft AR-DPCBF controller solves
\begin{eqnarray}
\label{eq:soft_qp}
\boldsymbol{u}^* = \arg\min_{\boldsymbol{u}\in\mathcal U}
\;
\|\boldsymbol{u}-\boldsymbol{u}_{\rm ref}\|^2
+ w \sum_{j=1}^{N_{\rm obs}}
\max\!\left(0,-g_j(\boldsymbol{x}, \boldsymbol{u})\right)^2 \nonumber
\\
\text{s.t.}\;
L_fh_j(\boldsymbol{x})
+
L_gh_j(\boldsymbol{x})\,\boldsymbol{u}
\ge
-\psi\!\left(h_j(\boldsymbol{x})\right),
\quad
\forall j,
\end{eqnarray}
where $h_j$ is the nominal DPCBF barrier and $w>0$ is the penalty weight.
\end{definition}

\begin{remark}
\label{rem:soft_zero_grad}
Soft AR-DPCBF does not satisfy the A-CBF condition of
Definition~\ref{def:acbf}, since violations of the adversarial constraint are permitted. Moreover, the penalty is inactive whenever $g_j(\boldsymbol{x}, \boldsymbol{u})>0$, providing no incentive to
increase the adversarial margin before the constraint becomes active. This motivates the proactive formulation introduced next.
\end{remark}

\subsection{Buffer Soft AR-DPCBF}
\label{subsec:buffer}

To enable proactive control, we replace the one-sided quadratic penalty with a buffered Huber penalty that becomes active before the certified adversarial margin is exhausted.

\begin{definition}[Buffer Soft AR-DPCBF]
\label{def:buffer}

For a buffer width $\varepsilon_b>0$, define
\begin{equation}
\label{eq:huber}
\phi(s;\varepsilon_b)=
\begin{cases}
0, & s>\varepsilon_b,\\[2pt]
{(\varepsilon_b-s)^2}/({2\varepsilon_b}),
&0<s\le\varepsilon_b,\\[8pt]
-s+({\varepsilon_b}/{2}),
&s\le0.
\end{cases}
\end{equation}
The function $\phi(\cdot;\varepsilon_b)$ is continuously differentiable on $\mathbb{R}$. The Buffer Soft AR-DPCBF controller solves
\begin{align}
\boldsymbol{u}^*
=
\arg\min_{\boldsymbol{u}\in\mathcal U}
\;&
\|\boldsymbol{u}-\boldsymbol{u}_{\rm ref}\|^2
+ w \sum_{j=1}^{N_{\rm obs}}
\phi\!\left(g_j(\boldsymbol{x},\boldsymbol{u});\varepsilon_b\right)
\label{eq:buffer_qp}
\\
\text{s.t.}\quad
&
L_fh_j(\boldsymbol{x})
+
L_gh_j(\boldsymbol{x})\,\boldsymbol{u}
\ge
-\psi\!\left(h_j(\boldsymbol{x})\right),
\quad
\forall j,
\nonumber
\end{align}
where $g_j$ is the robust residual \eqref{eq:robust_residual}, $h_j$ is the nominal DPCBF barrier, $w>0$ is the penalty weight, and $\varepsilon_b>0$ is the buffer width.
\end{definition}

\begin{remark}
\label{rem:buffer}
The buffered penalty has three operating regions. For $g_j>\varepsilon_b$, the penalty is inactive, and the controller coincides with the nominal DPCBF. Within the buffer,
$0<g_j\le\varepsilon_b$, the quadratic branch provides a nonzero steering signal before the adversarial condition is violated, encouraging proactive recovery of the certified margin. For $g_j\le0$, the linear branch limits the restoring effort, improving numerical conditioning while continuing to penalize violations. As $\varepsilon_b\rightarrow0$, \eqref{eq:buffer_qp} reduces to the Soft AR-DPCBF formulation.
\end{remark}

\begin{proposition}
\label{prop:soft_buffer_feasible}
The feasible sets of both the Soft~\eqref{eq:soft_qp} and Buffer AR-DPCBF
\eqref{eq:buffer_qp} coincide with that of the nominal DPCBF controller. Consequently, both formulations are feasible whenever the nominal DPCBF QP is feasible.
\end{proposition}

\begin{proof}
In both formulations, the adversarial robustness term appears only in the objective, while the constraints are identical to those of the nominal DPCBF controller. Hence the feasible set is unchanged.
\end{proof}






\begin{figure*}[t]
  \centering
  \includegraphics[width=0.95\textwidth]{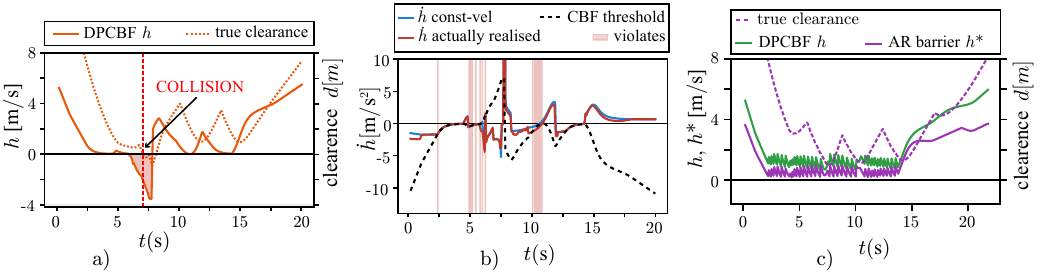}
 \caption{Illustration of the silent-failure mechanism under maneuvering obstacles. 
(a)~The nominal DPCBF certificate is violated, resulting in collision.
(b)~Silent failure: although the QP satisfies the nominal CBF constraint based on the assumed obstacle dynamics, the actual barrier derivative $\dot h$ falls below $-\psi(h)$ due to the unmodeled obstacle maneuver; shaded regions indicate time intervals during which the nominal CBF condition is satisfied while the corresponding true condition is violated. 
(c)~Under the same scenario, Buffer AR-DPCBF maintains both $h$ and $h^*$ positive throughout the trajectory.}
  \label{fig:silent}
\end{figure*}

\begin{table}[t]
\caption{Simulation parameters inherited
from~\cite[Table~I]{park2026dpcbf}. Bold entries are specific to the adversarial extension.}
\label{tab:params}
\centering
\footnotesize
\begin{tabularx}{\columnwidth}{l c X}
\toprule
Parameter & Value & Source / Reason \\
\midrule
$\ell_r$ & $0.20$ m & Rear-axle length \\
$a_{\max}$ & $5.0$ m\,s$^{-2}$ & Longitudinal acceleration limit \\
$\beta_{\max}$ & $0.28$ rad & Slip-angle limit \\
$v_{\max}$ & $3.5$ m\,s$^{-1}$ & Maximum speed \\
$v_{\mathrm{des}}$ & $2.5$ m\,s$^{-1}$ & Cruise reference speed \\
$r$ & $1.0$ m & Combined safety radius \\
$R_{\mathrm{sense}}$ & $15$ m & Sensing radius \\
$k_\lambda$ & $0.144$ & DPCBF gain \\
$k_\mu$ & $0.505$ & DPCBF gain \\
\midrule
$\boldsymbol{v_{\min}}$ & $\boldsymbol{1.0}$ m\,s$^{-1}$ &
Positive speed floor required for steering authority
(Assumption~\ref{as:rob}); gives $c_{\min}=1.40$ \\
$\boldsymbol{\gamma}$ & $\boldsymbol{3.0}$ & Pre-emption gain \\
$\boldsymbol{w}$ & $\boldsymbol{10.0}$ & Penalty weight for soft variants \\
$\boldsymbol{\varepsilon_b}$ & $\boldsymbol{0.3}$ & Buffer width (Buffer AR-DPCBF) \\
\bottomrule
\end{tabularx}
\end{table}

\section{Results}
\label{sec:results}

We evaluate the proposed AR-DPCBF framework through a series of simulation studies designed to validate the theoretical guarantees, characterize its robustness to adversarial obstacle maneuvers, and assess its practical performance.

\subsection{Simulation Setup}
\label{sec:setup}

\paragraph{Platform and scenario}
The robot is the kinematic bicycle~\eqref{eq:bicycle} with the parameters of
Table~\ref{tab:params}. The adversarial extension requires a positive speed
floor because steering authority scales as $v^2/\ell_r$ and vanishes at rest. We
set $v_{\min}=1.0$\,m\,s$^{-1}$, giving
$c_{\min}=\min\{a_{\max},v_{\min}^2\beta_{\max}/\ell_r\}=1.40$. The robot
navigates from $(2,12)$ to a goal at $x=50$\,m through $N$ obstacles in a
$62\times24$\,m arena, integrating at $\Delta t=0.06$\,s with RK4 over a $40$\,s
horizon. The representative operating point is $N=10$, $\nadv=5$,
$\kappa=0.98$\,m\,s$^{-2}$ (split as $a_{\mathrm{obs,max}}=0.5$,
$\omega_{\mathrm{obs,max}}=0.4$, $v_{\mathrm{obs,nom}}=1.2$), comfortably inside
the validity envelope $\kappa<c_{\min}$.

\paragraph{Obstacle behavior model}
All obstacles are unicycle agents with radii sampled from
$\mathcal U(0.1,0.7)$\,m. At each control step the $\nadv$ obstacles \emph{currently} closest to the robot are designated adversarial, assigned capability $\kappa=\Kmax$, and controlled to minimize the instantaneous barrier rate. The designation is recomputed every cycle, so obstacles switch roles as the
robot advances. The remaining $(N-\nadv)$ obstacles execute smooth motions with bounded accelerations and turn rates drawn from the same capability set
$\mathcal F$. Every obstacle therefore violates DPCBF's constant-velocity premise, while only the nearest $\nadv$ actively pursue the worst-case interaction.

\paragraph{Baselines}
Besides DPCBF and the three AR variants (Hard, Soft, Buffer), we compare against a robust-CBF baseline, \mbox{R-DPCBF}, which receives the identical capability set $\mathcal F$ but applies it pointwise to the barrier \emph{rate} instead of contracting the barrier \emph{parameters}. Its construction is given in Appendix~\ref{app:rdpcbf}.

\paragraph{Metrics}
\emph{Barrier violation}, the event $\min_t h(\boldsymbol{x}(t))<0$, is the primary metric. It is exactly the silent-failure event of Remark~\ref{rem:gap} and is detectable
even when no collision occurs. \emph{Collision} is the geometric event $d(t)<0$, and \emph{QP infeasibility} the per-cycle rate at which a controller's safety program admits no input, normalized by executed steps.  The adversary-count and capability sweeps use
$n_{seed}=100$ seeds per condition, the remaining sweeps $n_{seed}=25$-$50$ as stated per figure. Uncertainty is $95\%$ bootstrap CIs ($4000$ resamples). All controllers see identical seeds and scenarios, so paired comparisons use the exact two-sided McNemar test on discordant pairs.

\subsection{Breakdown of the Nominal Safety Certificate}
\label{sec:core}

\paragraph{Anatomy of a silent failure}
Figure~\ref{fig:silent} isolates the failure mechanism on a deterministic scenario ($N=10$, $\nadv=5$). DPCBF loses the safety certificate ($\min_t h=-3.53$) and collides (Fig.~\ref{fig:silent}(a)). The failure is \emph{silent}: although the QP satisfies its nominal CBF constraint on $92\%$ of control cycles, the true barrier rate falls below $-\psi(h)$ on $6.9\%$ of them because of the uncompensated disturbance $\Delta(t)$ (Fig.~\ref{fig:silent}(b)). Thus, the controller continues enforcing the nominal safety condition even though the true certificate has already been violated, rather than becoming infeasible. On the identical scenario, Buffer AR-DPCBF maintains both barriers positive throughout ($\min_t h=0.64$, $\min_t h^*=0.18$) and preserves a minimum clearance of $0.93$ m (Fig.~\ref{fig:silent}(c)).

\begin{figure*}[t]
  \centering
  \includegraphics[width=\textwidth]{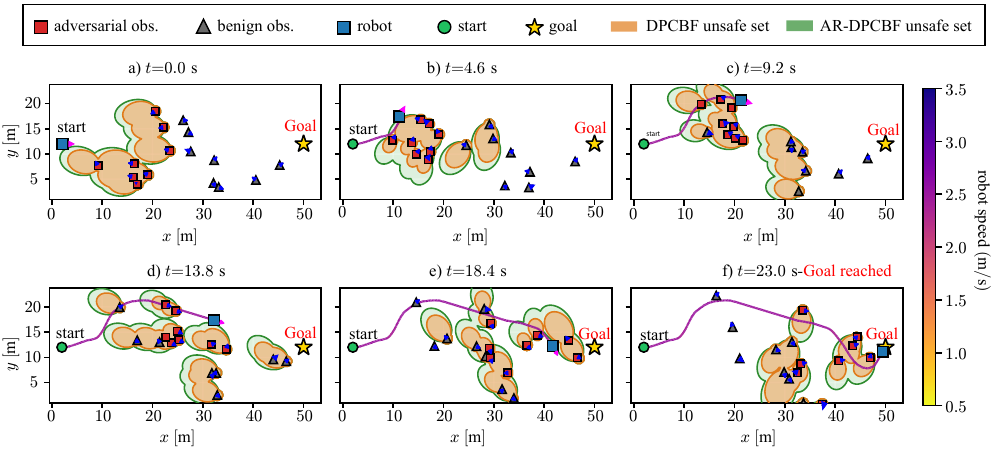}
  \caption{Buffer Soft AR-DPCBF navigation timelapse ($\min_t h^*=+0.21$). The unsafe sets are drawn with their exact $\{h=0\}$   (orange) and $\{h^*=0\}$ (green) boundaries; the green set is the uniformly  wider adversarial widening. The proactive buffer keeps the robot outside the contracted set throughout.}
  \label{fig:tl_buf}
\end{figure*}

\paragraph{Qualitative navigation}
Figure~\ref{fig:tl_buf} shows Buffer AR-DPCBF traversing a dense field with the
exact $\{h=0\}$ and $\{h^*=0\}$ boundaries drawn. The contracted set is the
uniformly wider adversarial widening of Theorem~\ref{thm:params}, and the
proactive buffer keeps the robot outside it for the entire run ($\min_t h^*=+0.21$). Soft AR-DPCBF on a comparable scenario also reaches the goal with $h>0$ throughout, but briefly enters the contracted set
($\min_t h^*=-0.08$) because its penalty activates only once $h^*<0$. The two variants, therefore, differ in \emph{when} they respond.

\paragraph{Adversarial fraction}
Figure~\ref{fig:adv} sweeps $\nadv\in\{0,\dots,15\}$ at fixed obstacle count $N=15$, isolating the effect of adversarial intent because every obstacle
maneuvers. Even at $\nadv=0$, DPCBF violates its certificate in $54\%$ of trials despite only $10\%$ collisions, exposing the silent-failure mode caused
by unmodeled obstacle maneuvers. As $\nadv$ increases, DPCBF violations rapidly saturate while collisions rise steadily, showing that the nominal certificate degrades under adversarial behavior. In contrast, all AR-DPCBF variants consistently reduce both metrics, preserving the ordering $\mathrm{Buffer}<\mathrm{Soft}<\mathrm{Hard}<\mathrm{DPCBF}$ across the entire sweep. Hard AR-DPCBF closely matches the pointwise-robustified baseline
R-DPCBF, demonstrating that robustifying the barrier derivative alone cannot replace contraction of the barrier geometry. Although enforcing the contracted barrier exactly reduces feasibility at high adversarial fractions, the soft variants recover near-nominal feasibility while retaining most of the robustness benefit. Goal attainment remains above $99\%$ for all methods, confirming that improved safety is achieved without sacrificing task completion.

\begin{figure*}[t]
  \centering
  \includegraphics[width=0.9\textwidth]{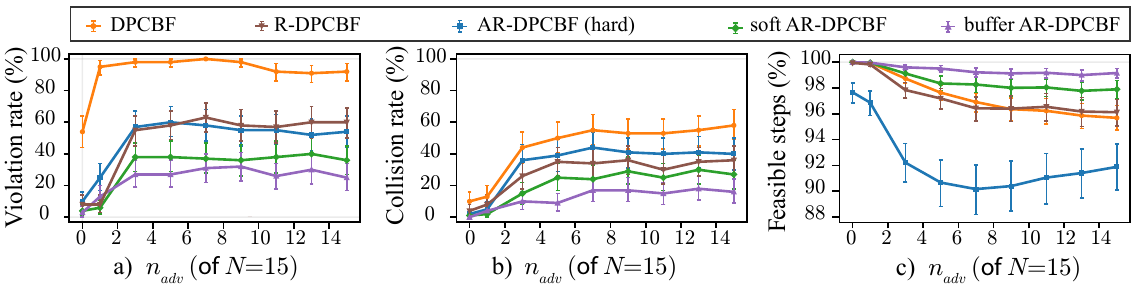}
  \caption{Safety versus adversarial fraction at high density ($N=15$,
  $\kappa=0.98$, $n=100$, $95\%$ bootstrap CIs). (a)~barrier violation, the
  silent-failure metric; (b)~collision, its downstream consequence; (c)~QP
  feasibility. The gap between (a) and (b) for DPCBF is the signature of the
  silent failure: the certificate is lost far more often than contact occurs.}
  \label{fig:adv}
\end{figure*}

\paragraph{Statistical significance}
A paired evaluation over the six densest scenarios ($N=15$,
$\nadv\in\{5,7,9,11,13,15\}$; $600$ paired trials) confirms that the safety
improvements observed in Fig.~\ref{fig:adv} are systematic rather than scenario-dependent. Relative to DPCBF, every AR variant reduces both collisions and barrier violations, with Buffer achieving the largest reduction, followed by Soft, Hard, and R-DPCBF. Exact two-sided McNemar tests show that all pairwise differences are statistically significant ($p<10^{-3}$), except Hard versus R-DPCBF on barrier violations ($p=0.041$), reflecting their similar performance. More importantly, the comparisons against DPCBF satisfies an even stronger property: across all $600$ paired trials, no AR variant violates the barrier on a seed where DPCBF remains safe. Thus, the improvement is not merely statistical but a strict containment relation.

\subsection{Robustness Across Adversary Capability}
\label{sec:capability}
We next evaluate performance under changes in the adversary's maneuvering capability. Specifically, we examine robustness under increasing obstacle capability, sensitivity to capability misestimation, and the tightness of the theoretical feasibility condition, thereby assessing both the practical robustness and the validity of the proposed analysis.

\paragraph{Robustness to Increasing Obstacle Capability}
Figure~\ref{fig:cap} fixes $n_{\mathrm{adv}}=5$ and sweeps the obstacle capability $\kappa\in[0,c_{\min}]$ ($N=10$, $n_{\mathrm{seed}}=100$). At $\kappa=0$, the adversarial contraction vanishes ($h^*\equiv h$), and DPCBF, Hard AR-DPCBF, and R-DPCBF produce identical outcomes on every seed, exactly as predicted by
Theorem~\ref{thm:params}(ii). As $\kappa$ increases, DPCBF rapidly loses both its safety certificate and collision avoidance performance. Hard AR-DPCBF and R-DPCBF substantially reduce barrier violations and
collisions relative to DPCBF, and exhibit comparable robustness, consistent with both controllers accounting for the same capability set $\mathcal F$. The soft AR-DPCBF variants consistently outperform both, yielding the ordering $\mathrm{Buffer}<\mathrm{Soft}<\{\mathrm{Hard}\approx
\mathrm{R\text{-}DPCBF}\}<\mathrm{DPCBF}$ throughout the sweep. 
The primary distinction is feasibility (Fig.~\ref{fig:cap}c). R-DPCBF retains near-nominal feasible-step rates comparable to DPCBF and the
soft variants by reverting to the nominal DPCBF program whenever its pointwise robust constraint becomes infeasible. In contrast, Hard AR-DPCBF progressively loses feasible QP solutions as $\kappa$ increases because it enforces the contracted barrier as a hard constraint. Consequently, while both Hard AR-DPCBF and R-DPCBF achieve similar robustness, the soft variants combine near-nominal feasibility with robustness gains beyond those attainable by R-DPCBF's pointwise treatment of the adversary.

\begin{figure*}[t]
  \centering
  \includegraphics[width=0.95\textwidth]{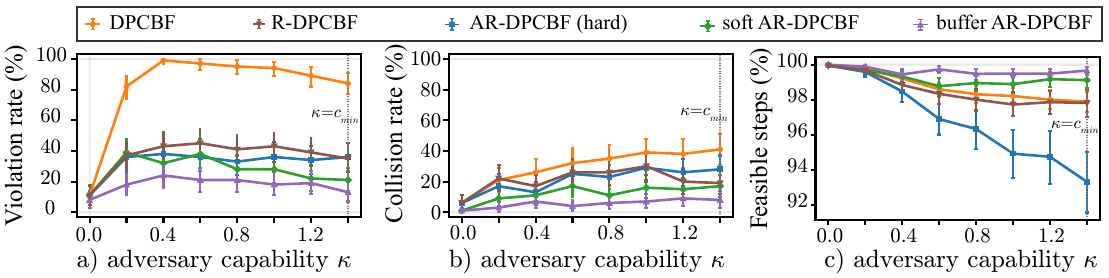}
  \caption{Safety versus adversary capability $\kappa$ ($N=10$, $\nadv=5$,  $n_{seed}=100$, $95\%$ bootstrap CIs). (a)~violation, (b)~collision, 
  (c)~QP  feasibility. DPCBF, Hard AR and R-DPCBF coincide exactly at $\kappa=0$  (Theorem~\ref{thm:params}(ii)). The dotted line is the validity limit
  $\kappa=c_{\min}=1.40$.}
  \label{fig:cap}
\end{figure*}

\paragraph{Robustness to Capability Misestimation}
Figure~\ref{fig:misspec} fixes the capability assumed by the controller at $\kappa_{\rm asm}=0.98$ and varies the obstacle's true capability, $\kappa_{\rm true}\in[0.3,2.0]$ ($n_{\mathrm{seed}}=25$), thereby isolating the effect of capability misestimation. When $\kappa_{\rm true}<\kappa_{\rm asm}$, the controller is conservative, but its safety performance is essentially unchanged. As $\kappa_{\rm true}>\kappa_{\rm asm}$, the obstacle becomes increasingly more capable than assumed, yet Buffer AR-DPCBF degrades gradually rather than exhibiting a failure threshold, substantially outperforming DPCBF even when the true capability exceeds the assumed value by more than a factor of two. Soft AR-DPCBF follows the same trend but loses robustness more rapidly. Thus, capability underestimation results in graceful degradation rather than abrupt loss of safety.

\begin{figure}[t]
  \centering
  \includegraphics[width=\columnwidth]{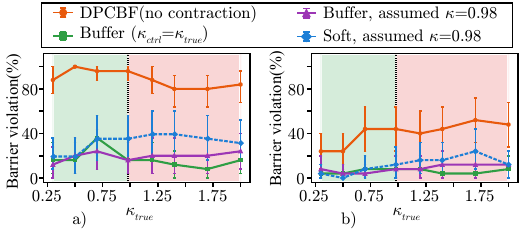}
  \caption{Robustness to capability mismatch ($n_{seed}=25$). (a)~violation,
  (b)~collision versus true capability $\kappa_{\mathrm{true}}$, with the  controller assuming $\kappa_{\mathrm{asm}}=0.98$\,m\,s$^{-2}$ except for the
  matched Buffer baseline $\kappa_{\mathrm{asm}} =\kappa_{\mathrm{true}}$).
  Green: over-provisioned; red: under-provisioned.}
  \label{fig:misspec}
\end{figure}

\paragraph{The feasibility condition is sound but conservative}
Theorem~\ref{thm:valid} guarantees maintainability of the adversarial certificate whenever $c_{\min}>\kappa$. To assess the tightness of this condition, we evaluate the worst-case Nagumo condition $\sup_{\boldsymbol{u}\in U}\inf_{\mathcal F}\dot h^*\ge0$ at sampled boundary states ($h^*=0$) with a single matched adversary, thereby isolating certificate maintainability from multi-obstacle interactions and closed-loop effects.
Figure~\ref{fig:mcnemar_bound}(b) shows that every sampled state within the certified region is maintainable, with no counterexamples, confirming the soundness of Theorem~\ref{thm:valid}. However, the empirical maintainability boundary lies well beyond the theoretical one: the transition, where roughly half of the sampled boundary states become unmaintainable, occurs near $c_{\min}\approx0.4$-$0.9$ even for $\kappa$ up to $2.5$. The conservatism, therefore, arises from the worst-case bounds $(\Lambda_{\min}, D_{\min})$ and the $1/\gamma$ barrier contraction rather than from the underlying geometry. Together with Fig.~\ref{fig:misspec}, these results show that the certificate remains maintainable well beyond its guaranteed region, while the closed-loop controller degrades gracefully once the guarantee no longer applies.

\begin{figure}[t]
  \centering
  \includegraphics[width=\columnwidth]{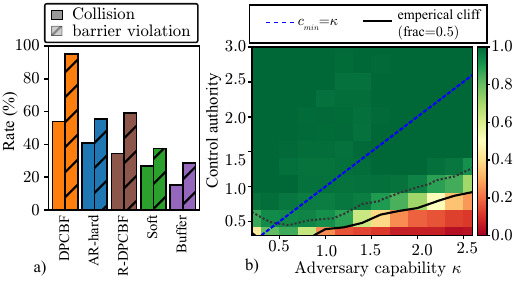}
\caption{(a)~Pooled collision (solid) and barrier-violation (hatched) rates across six dense conditions ($N=15$, $600$ paired trials). (b)~Empirical validation of Theorem~\ref{thm:valid}: colour indicates the fraction of boundary states satisfying the worst-case Nagumo condition. The region above the dashed line, $c_{\min}=\kappa$, is certified as uniformly maintainable, while the empirical feasibility cliff occurs beyond this threshold.}
  \label{fig:mcnemar_bound}
\end{figure}

\subsection{Mechanistic and Practical Validation}
\label{sec:rcbf}
This section examines the practical cost of its added robustness and its extension to unknown obstacle capabilities via online estimation.

\paragraph{Necessity of Barrier Contraction}
R-DPCBF isolates the role of barrier contraction by using the same capability set $\mathcal{F} $ as AR-DPCBF while applying robustness only to the instantaneous barrier derivative. Across all experiments, it consistently outperforms DPCBF but remains well below the soft AR-DPCBF variants, demonstrating that pointwise robustification alone cannot recover the robustness achieved by contracting the barrier geometry. The distinction is structural rather than numerical: pointwise worst-case guarantees $\dot h\ge0$ only at the current state, whereas barrier contraction reserves safety margin over the entire engagement, preventing the system from entering states in which no admissible control can maintain the certificate. This is precisely the mechanism established by Theorem~\ref{thm:params}.

R-DPCBF nevertheless illustrates an important trade-off. Because it leaves the nominal barrier unchanged, it preserves a larger feasible control set and achieves lower computational overhead than Hard AR-DPCBF. However, this efficiency comes at the expense of robustness: its collision performance plateaus as obstacle aggressiveness increases, whereas the contracted soft barriers continue to improve safety. Thus, the limitation of pointwise robustification is not computational efficiency but the absence of an anticipatory margin, highlighting barrier contraction as the essential ingredient for adversarial robustness.

\paragraph{The Price of Robustness}
\label{sec:price}

Figure~\ref{fig:frontier} illustrates the trade-off between safety and
efficiency by plotting collision rate against path overhead for $\nadv=5$ and $\nadv=15$. As expected, increased robustness incurs additional path length because larger safety margins require more conservative trajectories. However, the increase in overhead is modest relative to the reduction in collisions, with Buffer providing the best safety-efficiency trade-off. Increasing the adversarial fraction shifts the frontier primarily toward higher path overhead while leaving collision rates largely unchanged, indicating that denser adversarial environments increase the cost of safe navigation more than they reduce controller effectiveness. Goal attainment remains essentially unchanged across all controllers.

\begin{figure}[t]
  \centering
  \includegraphics[width=\columnwidth]{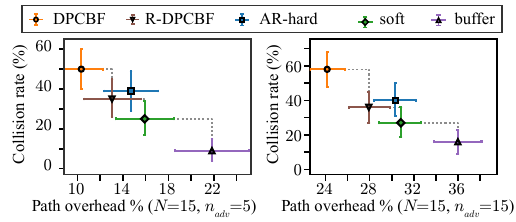}
  \caption{Safety-performance frontier at $N=15$ for $\nadv=5$ (left) and  $\nadv=15$ (right), $n_{seed}=100$. Path overhead is computed on the paired subset of seeds on which every controller is collision-free.}
  \label{fig:frontier}
\end{figure}

\paragraph{Sliding-Window Capability Estimation}
\label{sec:estimated}

The preceding experiments assume a known obstacle capability. Figure~\ref{fig:closedloop} instead estimates $\kappa$ online with the
sliding-window estimator of Section~\ref{subsec:estimator}, supplying the estimate to the controller each cycle and comparing against an oracle using the true value. Across all noise levels, the estimated-capability controller keeps violations and collisions below $17\%$ and $7\%$, respectively, substantially outperforming DPCBF ($93\%$, $43\%$) and matching or exceeding the oracle. The improvement arises because the estimator upper-bounds the true capability with empirical coverage $P(\tilde\kappa\ge\kappa)=0.95$-$0.99$, adding conservatism under measurement uncertainty; as noise decreases, the margin vanishes and the estimator approaches the oracle, consistent with Proposition~\ref{prop:robust}. The estimator is, therefore, a practical mechanism for deploying AR-DPCBF when obstacle capabilities are unknown.

\begin{figure}[t]
  \centering
  \includegraphics[width=\columnwidth]{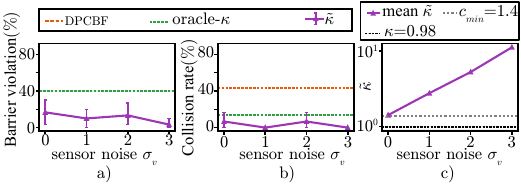}
  \caption{Closed-loop safety with the online-estimated capability
  (Section~\ref{subsec:estimator}, Proposition~\ref{prop:robust}; Soft  AR-DPCBF, $n_{seed}=30$). (a)~violation and (b)~collision versus sensor noise, with  DPCBF and oracle-$\kappa$ as references; (c)~the estimate $\tilde\kappa$  over-covers the true $\kappa$ with coverage  $P[\tilde\kappa\ge\kappa]\ge0.95$.}
  \label{fig:closedloop}
\end{figure}

\section{Conclusion}
\label{sec:conc}

We identified a structural limitation of CBF-based dynamic obstacle avoidance: under the common constant-velocity assumption, the safety certificate can remain satisfied even as maneuvering obstacles violate it. To address this, we proposed AR-DPCBF, which models obstacles as bounded-authority adversaries and reserves a certified safety margin through a state-dependent contraction of the DPCBF certificate. Exploiting the structure of the LoS dynamics yields a closed-form adversarial barrier, explicit feasibility conditions relating robot and obstacle capabilities, and practical controller formulations that recover
feasibility in dense environments and accommodate unknown obstacles
capabilities through online estimation. Extensive simulations validate the proposed framework across varying adversarial fractions, obstacle capabilities, and capability mismatch, demonstrating consistent elimination of the silent-failure mode of nominal DPCBF. Comparisons with a pointwise robust baseline further establish that contracting the barrier geometry, rather than merely robustifying its instantaneous derivative, is the essential mechanism for anticipatory safety against maneuvering obstacles.

Future work includes deriving controllers that jointly guarantee safety and feasibility, extending the framework to richer vehicle and obstacle dynamics.


\bibliographystyle{ieeetr}
\bibliography{citation.bib}

@inproceedings{park2026dpcbf, 
	  author    = {Park, Hun Kuk and Kim, Taekyung and Panagou, Dimitra},
	  title     = {Beyond Collision Cones: Dynamic Obstacle Avoidance for Nonholonomic Robots via Dynamic Parabolic Control Barrier Functions},
    booktitle = {IEEE International Conference on Robotics and Automation (ICRA)},
    shorttitle = {DPCBF},
    year      = {2026}
}

@article{ames2017cbf,
  author  = {Ames, Aaron D. and Xu, Xiangru and Grizzle, Jessy W. and Tabuada, Paulo},
  title   = {Control Barrier Function Based Quadratic Programs for Safety Critical Systems},
  journal = {IEEE Transactions on Automatic Control},
  volume  = {62}, number = {8}, pages = {3861--3876}, year = {2017},
  doi     = {10.1109/TAC.2016.2638961}
}

@inproceedings{xiao2019hocbf,
  author    = {Xiao, Wei and Belta, Calin},
  title     = {Control Barrier Functions for Systems with High Relative Degree},
  booktitle = {Proc. IEEE 58th Conf. on Decision and Control (CDC)},
  pages     = {474--479}, year = {2019},
  doi       = {10.1109/CDC40024.2019.9029455}
}

@article{xiao2022hocbf,
  author  = {Xiao, Wei and Belta, Calin},
  title   = {High-Order Control Barrier Functions},
  journal = {IEEE Transactions on Automatic Control},
  volume  = {67}, number = {7}, pages = {3655--3662}, year = {2022},
  doi     = {10.1109/TAC.2021.3105491}
}

@article{fiorini1998vo,
  author  = {Fiorini, Paolo and Shiller, Zvi},
  title   = {Motion Planning in Dynamic Environments Using Velocity Obstacles},
  journal = {The International Journal of Robotics Research},
  volume  = {17}, number = {7}, pages = {760--772}, year = {1998},
  doi     = {10.1177/027836499801700706}
}

@inproceedings{tayal2024c3bf,
  author    = {Goswami, Bhavya Giri and Tayal, Manan and Rajgopal, Karthik and
               Jagtap, Pushpak and Kolathaya, Shishir},
  title     = {Collision Cone Control Barrier Functions: Experimental Validation
               on {UGV}s for Kinematic Obstacle Avoidance},
  booktitle = {Proc. American Control Conference (ACC)},
  pages     = {325--331}, year = {2024},
  doi       = {10.23919/ACC60939.2024.10644338}
}

@inproceedings{bansal2017hj,
  author    = {Bansal, Somil and Chen, Mo and Herbert, Sylvia and Tomlin, Claire J.},
  title     = {{Hamilton-Jacobi} Reachability: A Brief Overview and Recent Advances},
  booktitle = {Proc. IEEE 56th Conf. on Decision and Control (CDC)},
  pages     = {2242--2253}, year = {2017},
  doi       = {10.1109/CDC.2017.8263977}
}

@article{fraichard2004ics,
  author    = {Fraichard, Thierry and Asama, Hajime},
  title     = {Inevitable Collision States --- A Step Towards Safer Robots?},
  journal   = {Advanced Robotics},
  volume    = {18}, number = {10}, pages = {1001--1024}, year = {2004},
  doi       = {10.1163/1568553042674662}
}

@article{xu2015robustcbf,
  author  = {Xu, Xiangru and Tabuada, Paulo and Grizzle, Jessy W. and Ames, Aaron D.},
  title   = {Robustness of Control Barrier Functions for Safety Critical Control},
  journal = {IFAC-PapersOnLine}, volume = {48}, number = {27}, pages = {54--61},
  year    = {2015}, doi = {10.1016/j.ifacol.2015.11.152}
}

@article{alan2023paramcbf,
  author  = {Alan, Anil and Molnar, Tamas G. and Ames, Aaron D. and Orosz, Gabor},
  title   = {Parameterized Barrier Functions to Guarantee Safety under Uncertainty},
  journal = {IEEE Control Systems Letters}, volume = {7}, pages = {2077--2082},
  year    = {2023}, doi = {10.1109/LCSYS.2023.3284857}
}

@inproceedings{polack2017kinematic,
  title={The kinematic bicycle model: A consistent model for planning feasible trajectories for autonomous vehicles?},
  author={Polack, Philip and Altch{\'e}, Florent and d'Andr{\'e}a-Novel, Brigitte and de La Fortelle, Arnaud},
  booktitle={2017 IEEE intelligent vehicles symposium (IV)},
  pages={812--818},
  year={2017},
  organization={IEEE}
}

@inproceedings{dacs2022robust,
  title={Robust safe control synthesis with disturbance observer-based control barrier functions},
  author={Da{\c{s}}, Ersin and Murray, Richard M},
  booktitle={2022 IEEE 61st conference on decision and control (CDC)},
  pages={5566--5573},
  year={2022},
  organization={IEEE}
}

@article{dacs2025robust,
  title={Robust control barrier functions using uncertainty estimation with application to mobile robots},
  author={Da{\c{s}}, Ersin and Burdick, Joel W},
  journal={IEEE Transactions on Automatic Control},
  volume={70},
  number={7},
  pages={4766--4773},
  year={2025},
  publisher={IEEE}
}

@article{quan2025observer,
  title={Observer-based environment robust control barrier functions for safety-critical control with dynamic obstacles},
  author={Quan, Ying Shuai and Zhou, Jian and Frisk, Erik and Chung, Chung Choo},
  journal={IEEE Control Systems Letters},
  year={2025},
  publisher={IEEE}
}

@inproceedings{kim2025robust,
  title={Robust control barrier function design for high relative degree systems: Application to unknown moving obstacle collision avoidance},
  author={Kim, Kwang Hak and Diagne, Mamadou and Krsti{\'c}, Miroslav},
  booktitle={2025 American Control Conference (ACC)},
  pages={355--360},
  year={2025},
  organization={IEEE}
}

@inproceedings{choi2021robust,
  title={Robust control barrier--value functions for safety-critical control},
  author={Choi, Jason J and Lee, Donggun and Sreenath, Koushil and Tomlin, Claire J and Herbert, Sylvia L},
  booktitle={2021 60th IEEE Conference on Decision and Control (CDC)},
  pages={6814--6821},
  year={2021},
  organization={IEEE}
}

@article{mitchell2005hji,
  title={A time-dependent Hamilton-Jacobi formulation of reachable sets for continuous dynamic games},
  author={Mitchell, Ian M and Bayen, Alexandre M and Tomlin, Claire J},
  journal={IEEE Transactions on automatic control},
  volume={50},
  number={7},
  pages={947--957},
  year={2005},
  publisher={IEEE}
}

@inproceedings{haraldsen2024safety,
  title={Safety-critical control of nonholonomic vehicles in dynamic environments using velocity obstacles},
  author={Haraldsen, Aurora and Wiig, Martin S and Ames, Aaron D and Pettersen, Kristin Y},
  booktitle={2024 American Control Conference (ACC)},
  pages={3152--3159},
  year={2024},
  organization={IEEE}
}

@article{huang2025dynamic,
  title={Dynamic collision avoidance using velocity obstacle-based control barrier functions},
  author={Huang, Jihao and Zeng, Jun and Chi, Xuemin and Sreenath, Koushil and Liu, Zhitao and Su, Hongye},
  journal={IEEE Transactions on Control Systems Technology},
  volume={33},
  number={5},
  pages={1601--1615},
  year={2025},
  publisher={IEEE}
}

@article{jankovic2018robust,
  author  = {Jankovic, Mrdjan},
  title   = {Robust Control Barrier Functions for Constrained Stabilization of
             Nonlinear Systems},
  journal = {Automatica}, volume = {96}, pages = {359--367}, year = {2018},
  doi     = {10.1016/j.automatica.2018.07.004}
}

@article{kolathaya2019issf,
  author  = {Kolathaya, Shishir and Ames, Aaron D.},
  title   = {Input-to-State Safety with Control Barrier Functions},
  journal = {IEEE Control Systems Letters}, volume = {3}, number = {1},
  pages   = {108--113}, year = {2019}
}

@article{alan2023issf,
  author  = {Alan, Anil and Taylor, Andrew J. and He, Chaozhe R. and
             Ames, Aaron D. and Orosz, G{\'a}bor},
  title   = {Control Barrier Functions and Input-to-State Safety with Application
             to Automated Vehicles},
  journal = {IEEE Transactions on Control Systems Technology},
  volume  = {31}, number = {6}, pages = {2744--2759}, year = {2023}
}

@article{nguyen2022robust,
  author  = {Nguyen, Quan and Sreenath, Koushil},
  title   = {Robust Safety-Critical Control for Dynamic Robotics},
  journal = {IEEE Transactions on Automatic Control},
  volume  = {67}, number = {3}, pages = {1073--1088}, year = {2022}
}

@article{breeden2023robust,
  author  = {Breeden, Joseph and Panagou, Dimitra},
  title   = {Robust Control Barrier Functions under High Relative Degree and Input
             Constraints for Satellite Trajectories},
  journal = {Automatica}, volume = {155}, pages = {111109}, year = {2023}
}

\appendices
\section{Pointwise Robustification Baseline (R-DPCBF)}
\label{app:rdpcbf}

As a baseline, we consider a robustified version of DPCBF that accounts for bounded obstacle maneuvers directly in the barrier-rate constraint. The controller assumes the same obstacle capability set $\mathcal F$ as AR-DPCBF but treats the unknown obstacle inputs as a
disturbance acting on $\dot h$.

Since the DPCBF barrier $h$ depends on the relative velocity, it has relative degree one with respect to the obstacle inputs. Its time derivative can therefore be decomposed as
\begin{equation}
\begin{aligned}
\dot h
&=
L_{\boldsymbol{f}} h
+
L_{\boldsymbol{g}} h\,\boldsymbol{u}
+
(\nabla_{\boldsymbol p_o} h)^\top \boldsymbol{v}_o
+
\Delta(\boldsymbol{x},a_o,\omega_o),
\end{aligned}
\label{eq:hdot_split}
\end{equation}
where $\Delta(\boldsymbol{x},a_o,\omega_o) = ({\partial h}/{\partial \boldsymbol{v}_o})a_o + ({\partial h}/{\partial\theta_o})\omega_o .$
For $(a_o,\omega_o)\in\mathcal F$, the disturbance satisfies $\Delta(\boldsymbol{x},a_o,\omega_o)\ge -m(\boldsymbol{x})$, where
$m(\boldsymbol{x}) =
a_{\mathrm{obs,max}}\|\nabla_{\boldsymbol v_o}h\|
+\omega_{\mathrm{obs,max}}|\partial_{\theta_o}h|.$
Substituting into $\dot h+\psi(h)\ge0$ yields,
\begin{equation}
L_{\boldsymbol{f}} h
+
L_{\boldsymbol{g}} h\,\boldsymbol{u}
+
(\nabla_{\boldsymbol p_o} h)^\top \boldsymbol{v}_o
-
m(\boldsymbol{x})
\ge
-\psi(h),
\label{eq:rdpcbf_constraint}
\end{equation}
which guarantees $\dot h\ge-\psi(h)$
for every admissible obstacle input. Since $m(\boldsymbol{x})$ depends only on the state, \eqref{eq:rdpcbf_constraint}
remains affine in the control input. The resulting controller therefore retains the same quadratic-program structure as DPCBF.

\section{Additional Theoretical Proofs}
\label{app:proofs}

\begin{lemma}\label{lem:Lmin}
Suppose~\eqref{eq:floor} holds, so that $\lambda^{*},\mu^{*}\ge0$ on
$\mathcal{B}=\partial\Cset\cap\Omega\cap\{d\ge r\}$
(Corollary~\ref{cor:floor}). Let
$\Sigma^{+}:=(k_{\mu}d_{\max})^{2}/(4v_{\mathrm{rel,min}}^{2})$. If
\begin{equation}
k_{\mu}\,d_{\max}<2\,v_{\mathrm{rel,min}},
\label{eq:Lmin-cond}
\end{equation}
then $\Lambda(\boldsymbol{x})\ge1-\Sigma^{+}>0$ for every $\boldsymbol{x}\in\mathcal{B}$, so
$\Lmin:=\inf_{\mathcal{B}}\Lambda\ge1-\Sigma^{+}>0$.
\end{lemma}

\begin{proof}
From \eqref{eq:Lambda} and the inequality
$\sqrt{a^{2}+b^{2}}\ge |a|$, we have
$\Lambda\ge |1-\sigma\vrx|$. It therefore suffices to obtain an upper
bound on $\sigma \vrx$. From \eqref{eq:mustar}, $\kappa d/(\gamma\nv)=k_{\mu}d-\mu^{*}$. Substituting this into
\eqref{eq:sigma} and multiplying by $\nv^{2}$ gives
$\sigma\nv^{2}
=
\lambda^{*}\vry^{2}
+\mu^{*}
-k_{\mu}d.$ On $\partial\Cset$, the boundary condition
$h^{*}=\vrx+\lambda^{*}\vry^{2}+\mu^{*}=0$ implies
$\lambda^{*}\vry^{2}+\mu^{*}=-\vrx$, and hence
$\sigma\nv^{2}
=
-\vrx-k_{\mu}d.$ Since $\lambda^{*},\mu^{*}\ge0$ on $\mathcal{B}$, the boundary condition also gives $\vrx=-(\lambda^{*} \vry^{2}+\mu^{*})\le0$.
Multiplying by $\vrx/\nv^{2}$ and introducing
$\chi:=\vrx/\nv$, we obtain
$\sigma\vrx
=
-\chi^{2}
-({k_{\mu}d}/{\nv})\chi,$
with $\chi\in[-1,0]$ because $|\vrx|\le\nv$ and $\vrx\le0$. The right-hand side is a concave quadratic in $\chi$ with stationary point
$\chi^{*}=-k_{\mu}d/(2\nv)$. Moreover, on $\mathcal{B}$ we have
$d\le d_{\max}$ and $\nv\ge v_{\mathrm{rel,min}}$, so
\[
\frac{k_{\mu}d}{2\nv}
\le
\frac{k_{\mu}d_{\max}}
     {2v_{\mathrm{rel,min}}}
<1
\]
by \eqref{eq:Lmin-cond}, hence $\chi^{*}\in(-1,0)$ and the maximum of
the quadratic over $[-1,0]$ is attained at $\chi^{*}$. Evaluating the
quadratic at $\chi^{*}$ yields
\[
\sigma\vrx
\le
\frac{1}{4}
\left(
\frac{k_{\mu}d}{\nv}
\right)^{2}
\le
\frac{1}{4}
\left(
\frac{k_{\mu}d_{\max}}
     {v_{\mathrm{rel,min}}}
\right)^{2}
=
\Sigma^{+}.
\]
Under \eqref{eq:Lmin-cond}, $\Sigma^{+}<1$, and therefore
$1-\sigma\vrx\ge1-\Sigma^{+}>0$. Consequently, $\Lambda
\ge
|1-\sigma\vrx|
=
1-\sigma\vrx
\ge
1-\Sigma^{+}.$ Taking the infimum over $\mathcal{B}$ yields
$\Lmin=\inf_{\mathcal{B}}\Lambda\ge1-\Sigma^{+}>0$.
\end{proof}

\end{document}